\documentclass[journal, letterpaper, 10pt]{IEEEtran}

\usepackage{times}
\usepackage{soul, cancel, xcolor}
 
\usepackage{algorithm}
\usepackage{algpseudocode}

\makeatletter
\def\therule{\makebox[\algorithmicindent][l]{\hspace*{.5em}\vrule height .75\baselineskip depth .25\baselineskip}}%

\newtoks\therules
\therules={}
\def\appendto#1#2{\expandafter#1\expandafter{\the#1#2}}
\def\gobblefirst#1{
 #1\expandafter\expandafter\expandafter{\expandafter\@gobble\the#1}}%
\def\LState{\State\unskip\the\therules}
\def\pushindent{\appendto\therules\therule}%
\def\popindent{\gobblefirst\therules}%
\def\printindent{\unskip\the\therules}%
\def\printandpush{\printindent\pushindent}%
\def\popandprint{\popindent\printindent}%

\algdef{SE}[WHILE]{While}{EndWhile}[1]
 {\printandpush\algorithmicwhile\ #1\ \algorithmicdo}
 {\popandprint\algorithmicend\ \algorithmicwhile}%
\algdef{SE}[FOR]{For}{EndFor}[1]
 {\printandpush\algorithmicfor\ #1\ \algorithmicdo}
 {\popandprint\algorithmicend\ \algorithmicfor}%
\algdef{S}[FOR]{ForAll}[1]
 {\printindent\algorithmicforall\ #1\ \algorithmicdo}%
\algdef{SE}[LOOP]{Loop}{EndLoop}
 {\printandpush\algorithmicloop}
 {\popandprint\algorithmicend\ \algorithmicloop}%
\algdef{SE}[REPEAT]{Repeat}{Until}
 {\printandpush\algorithmicrepeat}[1]
 {\popandprint\algorithmicuntil\ #1}%
\algdef{SE}[IF]{If}{EndIf}[1]
 {\printandpush\algorithmicif\ #1\ \algorithmicthen}
 {\popandprint\algorithmicend\ \algorithmicif}%
\algdef{C}[IF]{IF}{ElsIf}[1]
 {\popandprint\pushindent\algorithmicelse\ \algorithmicif\ #1\ \algorithmicthen}%
\algdef{Ce}[ELSE]{IF}{Else}{EndIf}
 {\popandprint\pushindent\algorithmicelse}%
\algdef{SE}[FUNCTION]{Function}{EndFunction}[2]
 {\printandpush\algorithmicfunction\ \textproc{#1}\ifthenelse{\equal{#2}{}}{}{(#2)}}%
 {\popandprint\algorithmicend\ \algorithmicfunction}
 \makeatother
 
\usepackage{cite, url, multicol, cancel}

\let\temp\rmdefault 
\usepackage{mathpazo}
\let\rmdefault\temp
\usepackage{tabularx,arydshln}
\usepackage{siunitx}
\protected\def\numpi{\text{\ensuremath{\pi}}}
\usepackage[pdftex, pdfauthor= {Giorgos Mamakoukas},
 pdftitle={Learning Stable Data-Driven Koopman Operators},
 pdfsubject={Online, Data-Driven Koopman-LQR control with Model Error Bounds with Application to Control of Robotic Fish in the presence of Fluid Disturbance},
 pdfkeywords={Koopman Operators, Data-Driven Control, Error Bounds, Robotic Fish},colorlinks = true,
 urlcolor = blue,]{hyperref}
\usepackage{booktabs}

\usepackage{flushend}
\usepackage{siunitx}
\usepackage{adjustbox}
\usepackage{dcolumn}
\newcolumntype{L}{D{.}{.}{2,5}}

\usepackage{xcolor, amsmath, amsfonts, subcaption, graphicx, mwe, tabularx}

\newcommand\scalemath[2]{\scalebox{#1}{\mbox{\ensuremath{\displaystyle #2}}}}

\usepackage{amsthm}
\newtheorem{theorem}{Theorem}
\newtheorem{definition}{Definition}

\newtheorem{assumption}{Assumption}
\newtheorem{proposition}{Proposition}

\usepackage{amssymb}
\newlength{\tempdima}

\usepackage{threeparttable}

\let\originalparagraph\paragraph
\renewcommand{\paragraph}[2][.]{\originalparagraph{#2#1}}

\title{\LARGE \bf Learning Stable Models for Prediction and Control 
}
		
\author{\fontsize{10}{13}\selectfont \authorblockN{Giorgos Mamakoukas$^*$,
Ian Abraham$^\dagger$, and
Todd D. Murphey$^*$}\\
\authorblockA{\fontsize{10}{13}\selectfont $^*$Department of Mechanical Engineering, Northwestern University, 
Evanston, Illinois 60208, USA}\\
\authorblockA{\fontsize{10}{13}\selectfont $^\dagger$The Robotics Institute, Carnegie Mellon University, 
Pittsburgh, Pennsylvania 15213, USA}\\
 \href{mailto:giorgosmamakoukas@u.northwestern.edu}{\fontsize{10}{13}\selectfont giorgosmamakoukas@u.northwestern.edu}}
\begin{document}
\maketitle
\IEEEpubid{\begin{minipage}{\textwidth}\ \vspace{4ex}\\[8pt] 
  \copyright 2022 IEEE.  Personal use of this material is permitted.  Permission from IEEE must be obtained for all other uses, in any current or future media, including reprinting/republishing this material for advertising or promotional purposes, creating new collective works, for resale or redistribution to servers or lists, or reuse of any copyrighted component of this work in other works.
\end{minipage}} 

\begin{abstract}
This paper demonstrates the benefits of imposing stability on data-driven Koopman operators. The data-driven identification of stable Koopman operators (DISKO) is implemented using an algorithm \cite{mamakoukas_stableLDS2020} that computes the nearest \textit{stable} matrix solution to a least-squares reconstruction error. As a first result, we derive a formula that describes the prediction error of Koopman representations for an arbitrary number of time steps, and which shows that stability constraints can improve the predictive accuracy over long horizons.  As a second result, we determine formal conditions on basis functions of Koopman operators needed to satisfy the stability properties of an underlying nonlinear system. As a third result, we derive formal conditions for constructing Lyapunov functions for nonlinear systems out of stable data-driven Koopman operators, which we use to verify stabilizing control from data. Lastly, we demonstrate the benefits of DISKO in prediction and control with simulations using a pendulum and a quadrotor and experiments with a pusher-slider system. The paper is complemented with a video: \url{https://sites.google.com/view/learning-stable-koopman}. 
\end{abstract}
\section{INTRODUCTION}
The dynamics of robots are often unknown or stochastic (e.g., Sphero SPRK \cite{Ian_KoopmanMPC}, soft robotics \cite{Bruder_Koopman, soft_robotics}) and environments that are complex or changing (such as sand \cite{Biomimemtic_on_sand, HumanoidRobots_onGrass_sands_rocks, Driving_on_sand} or water \cite{Underwater_navigation_challenges,mamakoukas2018feedback, modeling_control_underwater, mamakoukas2016sequential}) are hard to model accurately. In the face of such uncertainty, robotic applications can often fail due to poor prediction and control. For this reason, system identification methods are used to develop or adapt a model from data \cite{koopman_sindy, Online_system_identification, Osama_adaptive, System_identification_manipulator, Identification_dynamic_model, Identification_humanoid}. 
To further improve the learning and quality of the identified dynamics, researchers are developing active learning methods to strategically maneuver a robot and actively collect measurements that will reduce the uncertainty of the dynamics and the environment \cite{Ian_active_learning, active_learning_inverse_models, active_learning_montecarlo, activelearning_mobile_robotics}. 

\subsection{Data-Driven Applications of Koopman Operators} 
Recent data-driven efforts have focused in particular on Koopman operators \cite{koopman}. Koopman operators are linear embeddings of nonlinear systems \cite{koopman_KIC, koopman_mezic, koopmanism, koopman_mpc} and have gained attention for the purposes of both system identification \cite{koopman_linear_si} and real-time nonlinear control \cite{koopman_mpc}. Besides simplicity, using the Koopman linear representation for control can in certain cases also outperform feedback policies that are based directly on the nonlinear dynamics \cite{brunton_invariant, koopman_kronic}. However, with few exceptions \cite{brunton_invariant, koopman_invariant_Naoya, haseli2019efficient}, Koopman operators are typically infinite-dimensional and studies seek finite-dimensional approximations Koopman operators using methods such as the Dynamic Mode Decomposition (DMD) \cite{DMD}, extended DMD (EDMD)\cite{koopman_datadrivenapproximation_edmd, Koopman_EDMD}, Hankel-DMD \cite{Hankel_DMD}, or closed-form solutions \cite{mamakoukas2021derivative, koopman_si}. 
\IEEEpubidadjcol

These methods have already been used for Koopman-based prediction and control in many applications, such as robotics \cite{mamakoukas2021derivative, Ian_KoopmanMPC, Bruder_Koopman, castano2020Koopman}, human locomotion \cite{applications_human_locomotion}, neuroscience \cite{applications_neuroscience}, fluid mechanics \cite{Applications_fluid_flows}, and climate forecast \cite{Koopman_longterm_prediction}. Due to the high dimensionality of the states or due to limited resources (memory or computational speed), researchers typically compute Koopman representations offline once and do not update them in real time \cite{Bruder_Koopman, RSS2019_MamakoukasCastano}. Without online updates to the model, it is important that the approximate Koopman operator captures properties of the underlying system to remain accurate throughout the state space. Further, it is important that it remains accurate for long time horizons to enable more farsighted control \cite{farSighted}. 
\IEEEpubidadjcol

\subsection{Learning Koopman Models: Challenges and Related Work}

Learning Koopman representations remains an open question, but only few studies have focused on optimizing the long-term accuracy of the data-driven models. Work in \cite{predictiveaccuracy_DMD} computes error bounds for the Dynamic Mode Decomposition, closely related to the Koopman operator, but the analysis is applicable only to systems with parabolic partial differential equations and has restrictive assumptions on the stability of the identified dynamics. Work in \cite{koopman_deeplearning} uses deep learning to identify Koopman eigenfunctions from data and considers a loss function that measures the long-term accuracy of the Koopman operator for an arbitrary number of steps into the future. However, the optimization complexity grows with the number of prediction steps. Given that deep learning requires large datasets, the proposed method is not data-efficient and scales poorly to high-dimensional systems or long prediction horizons.

There is only one study that has so far imposed physics-based properties on the Koopman operator. Work in \cite{Koopman_dissipativity} learns Koopman operators under dissipativity constraints and is the first study that tries to exploit \textit{a priori} information about the nonlinear dynamics. However, the analysis assumes that the dissipativity properties of the system are given, which is not always the case. Further, the authors do not discuss restrictions that dissipative Koopman operators place on the basis functions. For example, inverse functions that diverge at the equilibrium are inconsistent with a dissipative representation and cannot be part of the learned Koopman model. Last, because the optimization problem is numerically intractable, the authors compute solutions to approximate objectives instead. 

Research efforts have typically overlooked whether the properties of the learned model, e.g., stability, are consistent with those of the original system. For example, work in \cite{Koopman_EDMD} makes assumptions on the stability of the underlying nonlinear dynamics---i.e., the system has a single attractor that is (asymptotically) stable---but does not enforce similar constraints on the learned model. As a result, stable nonlinear dynamics are sometimes represented by Koopman operators that are unstable, due to noise, poor quality (e.g., sparse or highly-correlated) measurements, or even limitations of the learning schemes used \cite{Robust_Koopman, SparseData_Koopman, Robust_Koopman_v2}. The need for stable Koopman operators is further highlighted with recent efforts that attempt to enforce stability with indirect methods \cite{sinha2021few} or require a stable Koopman operator for safety-critical control via the synthesis of control barrier functions \cite{folkestad2020data}. Needless to say, when the stability properties of the underlying system and the learned model do not match, the Koopman-based evolution of the states diverges exponentially from the true solution. At the same time, model predictive control with long planning horizon requires models that are accurate over a long time. 

\subsection{Contribution and Structure}
This work considers the data-driven identification of stable Koopman operators (DISKO) for the purposes of predictions that remain numerically stable and accurate over long time horizons, as well as stabilizing control. Specifically, we 
\begin{itemize}
 \item derive the prediction error induced by Koopman models over an arbitrary number of time steps (a result that is used to show the need for stable operators),
 \item provide conditions for choosing Koopman basis functions that are consistent with the stability properties of the underlying nonlinear system and which can improve data-driven learning, 
 \item present a method to construct Lyapunov functions for nonlinear dynamics, which are used to verify stabilizing controllers, and
\item demonstrate the benefits of DISKO for prediction and control in both simulations and experiments.
\end{itemize} 

This paper is structured as follows. Section \ref{sec:: Koopman_Operator} reviews the Koopman Operator framework for prediction and control. Section \ref{sec:: Stability_properties_of_Koopman} i) derives a formula for the prediction error of arbitrary time steps for Koopman representations, ii) presents conditions for admissible basis functions of a Koopman operator that is consistent with the stability properties of underlying nonlinear dynamics, and iii) derives the conditions and methodology for constructing Lyapunov functions for nonlinear systems using stable approximate Koopman operators. Section \ref{sec::DISKO} introduces the data-driven identification of stable Koopman operators (DISKO). Section \ref{sec:: Results} demonstrates the benefits of DISKO in prediction and control of nonlinear systems with simulation and experimental results. Section \ref{sec:: Discussion} summarizes the findings and discusses areas that merit further investigation.

\section{Koopman Operator}\label{sec:: Koopman_Operator}
The Koopman operator $\mathcal{K}$ linearly evolves functions of the states $s(t) \in \mathcal{S} \subseteq \mathbb{R}^N$ (i.e. $\Psi(s(t))$, commonly referred to as observables or basis functions) without loss of accuracy \cite{koopman}. Given general nonlinear dynamics of the form 
\begin{align}
 s(t_k + \Delta t)  = F(s(t_k)),
\end{align}
where $F$ is the flow map, the Koopman operator advances the observables with the flow of the dynamics:
\begin{align}
 \mathcal{K}\Psi = \Psi \circ F.
\end{align}
Formally, the continuous- and discrete-time operators are respectively given by
\begin{equation*}
\resizebox{0.99\hsize}{!}{$
\begin{aligned}
 \frac{d}{dt} \Psi(s(t)) = \mathcal{K} \Psi(s(t)) \quad \text{and} \quad \Psi(s(t_k + \Delta t)) = \mathcal{K}_d \Psi(s(t_k)),
\end{aligned}
$}
\end{equation*}
where the two operators are linked via $\mathcal{K} = \log(\mathcal{K}_d)/\Delta t$\cite{antsaklis2006linear} (Chapter 2, p. 156). Although a linear representation, the Koopman operator evolves nonlinear dynamics with full fidelity throughout the state space, contrary to methods that locally linearize dynamics around a point or a trajectory. For a more comprehensive review of the Koopman operator, we refer the reader to \cite{koopmanism}.

\subsection{Data-Driven Approximations of Koopman Operators}
With few exceptions \cite{brunton_invariant}, Koopman operators are infinite-dimensional and recent studies focus on obtaining finite-dimensional approximations for the purposes of system identification and control \cite{Koopman_system_identification, Ian_KoopmanMPC, Bruder_Koopman}. There are several methods to approximate Koopman operators, such as DMD \cite{DMD}, EDMD \cite{koopman_datadrivenapproximation_edmd, Koopman_EDMD}, Hankel-DMD \cite{Hankel_DMD}, closed-form solutions \cite{RSS2019_MamakoukasCastano, koopman_si}, or regression techniques, such as Least Absolute Shrinkage and Selection Operator (LASSO) regression \cite{LASSO, Bruder_Koopman}. In this paper, we consider the least-squares solutions of the local one-time-step error across $P$ measurements given by
\begin{equation}\label{eq:: Kd_LS}
\begin{aligned}
\tilde{\mathcal{K}}^*_d = \underset{\tilde{\mathcal{K}}_d}{\operatorname{argmin}}&~ \sum_{k = 1}^{P}\frac{1}{2}\lVert \Psi(s(t_k + \Delta t), u(t_k + \Delta t)) \\ &- \tilde{\mathcal{K}}_d \Psi(s(t_k), u(t_k))\rVert^2,
\end{aligned}
\end{equation}
where $\tilde{\mathcal{K}}_d \in \mathbb{R}^{W\times W}$ is the finite-dimensional approximation of the Koopman operator and $\Psi(s(t_k)): \mathbb{R}^N \mapsto \mathbb{R}^W$ are the observables.
Each measurement $k$ consists of the initial state $s(t_k)$, final state $s(t_k + \Delta t)$ and the actuation applied at the same instants, $u(t_k)$ and $u(t_k+\Delta t)$, respectively. Equation \eqref{eq:: Kd_LS} has a closed-form solution given by
\begin{align}\label{eq:: Kd_AG}
\tilde{\mathcal{K}}^*_d = \mathcal{A}\mathcal{G}^\dagger,
\end{align}
with 
\begin{align*}
\mathcal{A} =&\sum_{k = 1}^{P} \Psi(s(t_k + \Delta t), u(t_k + \Delta t)) \Psi(s (t_k), u(t_k))^T \intertext{and}
\mathcal{G} =& \sum_{k = 1}^{P} \Psi(s(t_k), u(t_k)) \Psi(s(t_k), u(t_k))^T
\end{align*}
where $^\dagger$ refers to the Moore-Penrose pseudoinverse. Note that the time spacing $\Delta t$ between measurements $s(t_k)$ and $s(t_k + \Delta t)$ must be consistent for all $P$ training measurements. 

Note that we use the $\sim$ notation to indicate the \textit{data-driven, approximate} value of the Koopman operator, and not its dimension. We use $\mathcal{K}$ to refer to the Koopman operator, which is typically infinite-dimensional, but can also be finite-dimensional (as is proven for some systems \cite{brunton_invariant}); we use $\tilde{\mathcal{K}}$ to refer to the finite-dimensional \textit{data-driven approximation} of the Koopman operator. On the other hand, we use $\Psi$ to refer to the vector of analytical functions that represent the basis functions of arbitrary (finite or infinite) dimension, associated with either an exact or approximate Koopman model.

Koopman operators offer an easily implementable system identification framework that is conducive to linear control tools for nonlinear dynamical systems. The linear representation makes it easier to analyze the properties of the underlying system, such as model accuracy or regions of attraction \cite{koopman_basisofattraction, koopman_stabilityanalysis}. Further, Koopman operator-based control can even outperform feedback that utilizes full knowledge of the nonlinear dynamics \cite{brunton_invariant, koopman_kronic}. 

\subsection{Control of Nonlinear Dynamics Using Koopman Operators}
Consider a linear system with states $s(t) \in \mathbb{R}^N$, control $u(t) \in \mathbb{R}^M$, and a discrete-time performance objective \vspace{-0.1cm}
\begin{align}\label{eq::objective}
J = \frac{1}{2}\sum_{t_k=0}^{\infty} \lVert s\left(t_k\right)-s_{des}\left(t_k\right)\rVert_Q^2 + \lVert u\left(t_k\right)\rVert_R^2,
\end{align}
where $Q \succeq 0 \in \mathbb{R}^{N\times N}$ and $R \succ 0 \in \mathbb{R}^{M \times M}$ are weights on the deviation from the desired states and the applied control, respectively. Next, we use the Koopman operator dynamics to design an equivalent objective function and a control response for the original nonlinear system. 

To simplify the analysis, we choose basis functions that depend separately on the states and control. That is, we consider basis functions $\Psi(s(t_k), u(t_k))$ = $[\Psi_s(s(t_k)), \Psi_{u} (u(t_k))]^T$, where $\Psi_s(s(t_k)) \in \mathbb{R}^{W_s}$ are the Koopman basis functions that depend only on the states, and $ \Psi_{u} (u(t_k)) \in \mathbb{R}^{W_u}$ are those that depend only on the input, such that $W = W_s + W_u$. Then, we can write the Koopman dynamics as 
\begin{align*}
\Psi(s(t_k + \Delta t), u(t_k + \Delta t)) =& \begin{bmatrix} \Psi_s(s(t_k + \Delta t)) \\ \Psi_{u} (u(t_k + \Delta t)) \end{bmatrix} \\\approx& 
\begin{bmatrix}
A & B \\
\cdot & \cdot
\end{bmatrix}
\begin{bmatrix} \Psi_s(s(t_k)) \\ \Psi_{u} (u(t_k)) \end{bmatrix},
\end{align*}
where $(\cdot)$ is used to indicate the terms associated with predicting the evolution of control, which is of no interest in this paper as it will be determined by the feedback policy. The terms $A \in \mathbb{R}^{W_s \times W_s}$ and $B \in \mathbb{R}^{W_s \times W_u}$ are sub-matrices of $\tilde{\mathcal{K}}_d$ and are fixed unless $\tilde{\mathcal{K}}_d$ is updated. Note that the dynamical equation has been modified to allow for control inputs \cite{koopman_KIC}. Here, we choose $\Psi_u(u(t_k)) = u(t_k)$ to simplify the dynamics for the purposes of LQR control.  The dynamics of the Koopman state-dependent basis functions are then
\begin{align}\label{eq:: AxBu_Koopman}
\Psi_s(s(t_k + \Delta t)) \approx 
A \Psi_s(s(t_k)) + B u(t_k).
\end{align}
In this paper, we always choose the system states as the first $N$ basis functions. Then, to retrieve the states from the Koopman prediction, one simply uses the first $N$ basis functions.

Given the Koopman representation, we write the discrete-time performance objective as
\begin{align}\label{eq:: J_K}
J_{\tilde{\mathcal{K}}} = \frac{1}{2}\sum_{t_k=0}^{\infty} & \lVert \Psi_s\left(s\left(t_k\right)\right) - \Psi_s\left(s_{des}\left(t_k\right)\right)\rVert_{Q_{\tilde{\mathcal{K}} }}^2 
+ \frac{1}{2}\lVert u(t_k)\rVert_R^2,
\end{align}
where $Q_{\tilde{\mathcal{K}} } \succeq 0 \in \mathbb{R}^{W_s \times W_s}$ penalizes the deviation from the desired observable functions $\Psi_s(s_{des}(t_k))$. We then use the Koopman representation to develop linear quadratic regulator (LQR) feedback of the form
\begin{align} \label{eq:: K_LQR}
u(t_k) = - K_{LQR} (\Psi_s(s(t_k)) - \Psi_s(s_{des}(t_k) )).
\end{align}
where $K_{LQR} \in \mathbb{R}^{W_u \times W_s}$ are the LQR gains.

\section{Stable Koopman Operators}\label{sec:: Stability_properties_of_Koopman}
In this section, we derive the error induced by approximate Koopman operators over an arbitrary number of time steps into the future. We then use the error expression to motivate imposing stability on the operators, sometimes even in cases when the underlying system is unstable. Then, we present conditions on the basis functions that are consistent with stable Koopman operators. Last, we demonstrate how to construct Lyapunov functions using stable data-driven Koopman operators. In the following analysis, we make use of the definitions for unstable, (marginally) stable, and asymptotically stable matrices.

\begin{definition} [Continuous-time: Theorem 8.1, \cite{hespanha2018linear}] A square matrix $A$ is i) unstable if and only if at least one of its eigenvalues has a positive real part or zero real part with a Jordan block of size larger than 1; ii) asymptotically stable or Hurwitz if and only if all the eigenvalues have strictly negative real parts; iii) (marginally) stable if and only if all of the eigenvalues have negative or zero real parts and all the eigenvalues with zero real parts are distinct.
\end{definition}

The analysis in Section \ref{sec:: Stability_properties_of_Koopman} discusses the properties of continuous-time Koopman models. The algorithm used to identify stable Koopman models uses state measurements and calculates a discrete-time representation by bounding the magnitude of the eigenvalues. For completeness, we include the definition of stable matrices in discrete-time as well. 

\begin{definition} [Discrete-time: Theorem 8.3, \cite{hespanha2018linear}] A square matrix $A$ is i) unstable if and only if at least one of its eigenvalues has magnitude greater than 1; ii) asymptotically stable or Schur if and only if all the eigenvalues have magnitude strictly less than 1; iii) (marginally) stable if and only if the maximum magnitude of its eigenvalues is 1 and the eigenvalues with magnitude equal to 1 are all distinct.
\end{definition}
While in \cite{hespanha2018linear} an asymptotically stable system includes (marginally) stable systems, in this work we consider the two stability types separately. That is, by (marginally) stable we will refer to matrices and systems that are neither unstable nor asymptotically stable. For the rest of this paper, we will refer to a system that is either (marginally) stable or asymptotically stable as stable to indicate that the system is not unstable. Note that if the response of a linear system is bounded, that is $\|s(t) \| \le c, \forall~t\ge0$ for some $c > 0$, then the system is stable (Definition 8.1 in \cite{hespanha2018linear}). As a result, a nonlinear system with a bounded response can be represented without loss of accuracy only by stable Koopman operators.
\subsection{Global Error of Approximate Koopman Operators} \label{ssec::global_error}
\begin{figure}
	\centering
	\includegraphics[width= \columnwidth, keepaspectratio]{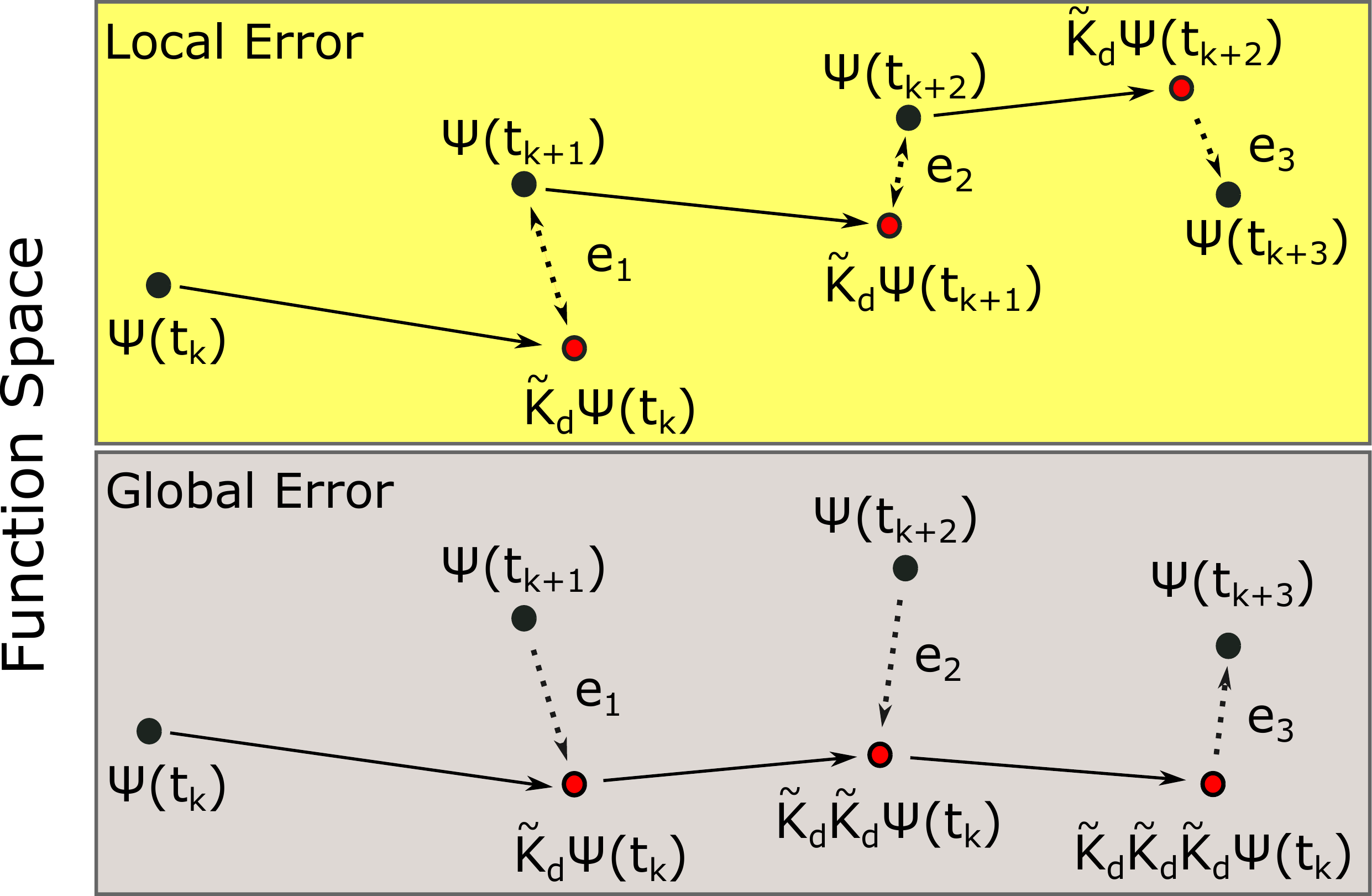} 
	\hfill%
	\caption{Local and global errors (in time) induced by approximate Koopman operators. The local error considers the accuracy of the model across a single time step; the global error considers the accuracy of the model across all time steps and is a more representative metric of long-term accuracy, assuming states are not updated in every time step.} 
	\label{fig:: Loval vs Global} 
\end{figure}

\subsubsection{Notation}
At time $t_0 + n \Delta t$, we use $\Psi(s(t_0 + n \Delta t))$ to indicate the true value of the basis function evaluated with the true state value $s(t_0 + n \Delta t)$ and $\tilde{\Psi}_n$ to indicate the approximate solution, where $n \in \mathbb{Z}^+ $ indicates the number of time steps into the future. We use $e_n$ to be the local error at the $n$th time step induced by the approximate Koopman operator $\tilde{\mathcal{K}}_d$, that is
\begin{align}\label{eq:: local_error}
e_{n} \equiv \Psi(s(t_0 + n \Delta t)) - \tilde{\mathcal{K}}_d \Psi(s(t_0 + (n-1) \Delta t)),
\end{align}
which assumes that the Koopman operator propagates the true value of the basis functions from the previous time step. 
Similarly, we use $E_n$ to refer to the global error at the $n$th time step:
\begin{align}\label{eq:: GE_simplified}
E_{n} \equiv & \Psi(s(t_0 + n \Delta t)) - \tilde{\mathcal{K}}_d^n \Psi(s(t_0)),
\end{align}
which, for $n=1$, matches the local error \eqref{eq:: local_error}. However, note that the global error \eqref{eq:: GE_simplified} is not an accumulation of the local errors \eqref{eq:: local_error}: $E_n \ne \sum_i^n e_i$. We illustrate the difference between local and global error in Fig. \ref{fig:: Loval vs Global}.

\subsubsection{Derivation of Global Error}
Consider the true solution at some time step $t_0 + n \Delta t$, which can be written using \eqref{eq:: local_error} as
\begin{align}\label{eq:: error_first_backward_step}
\Psi(s(t_0 + n \Delta t)) = \tilde{\mathcal{K}}_d \Psi(s(t_0 + (n-1) \Delta t)) + e_{n}. 
\end{align}
Similarly, 
\begin{align}\label{eq:: error_second_backward_step}
\Psi(s(t_0 + (n-1) \Delta t)) = \tilde{\mathcal{K}}_d \Psi(s(t_0 + (n-2) \Delta t)) + e_{n-1}. 
\end{align}
Plugging \eqref{eq:: error_second_backward_step} into \eqref{eq:: error_first_backward_step}
\begin{equation}
\resizebox{0.99\hsize}{!}{$
\begin{aligned}
\Psi(s(t_0 + n \Delta t)) =&~\tilde{\mathcal{K}}_d \big(\tilde{\mathcal{K}}_d \Psi(s(t_0 + (n-2) \Delta t)) + e_{n-1}\big) + e_{n} \\
 =&~\tilde{\mathcal{K}}_d^2 \Psi(s(t_0 + (n-2) \Delta t)) + \tilde{\mathcal{K}}_d e_{n-1} + e_{n}. 
\end{aligned}
$}
\end{equation}
Recursively expressing the solution in terms of the true values of the basis functions in the previous steps and the corresponding local error yields
\begin{align*}
\Psi(s(t_0 + n \Delta t)) =& \tilde{\mathcal{K}}_d^n \Psi(s(t_{0})) + \sum_{i=0}^{n-1} \tilde{\mathcal{K}}_d^i e_{n-i}.
\end{align*} 
Therefore, using \eqref{eq:: GE_simplified}, the global error at $t_0 + n \Delta t$ is 
\begin{align*}
E_{n} = \sum_{i=0}^{n-1} \tilde{\mathcal{K}}_d^i e_{n-i}.
\end{align*}
\subsubsection{Global Error Bound}
Next, we compute an upper bound for the global error. We use an induced norm that satisfies the properties of triangle inequality ($
\lVert A + B \rVert \le \lVert A\rVert + \lVert B\rVert$), subordinance ($\lVert Ax \rVert \le \lVert A \rVert \lVert x \rVert$) and submultiplicativity ($
\lVert AB \rVert \le \lVert A\rVert \cdot \lVert B\rVert
)$, such that
\begin{align*}
 \lVert E_{n} \rVert =&~ \lVert \sum_{i=0}^{n-1} \tilde{\mathcal{K}}_d^i e_{n-i}\rVert \\
 \le& \sum_{i=0}^{n-1} \lVert \tilde{\mathcal{K}}_d^i e_{n-i}\rVert & \tag{triangle inequality}\\
 \le& \sum_{i=0}^{n-1} \lVert\tilde{\mathcal{K}}_d^i\rVert \cdot \lVert e_{n-i}\rVert &\tag{subordinance}\\
 \lVert E_{n} \rVert \le& \sum_{i=0}^{n-1} \lVert\tilde{\mathcal{K}}_d\rVert^i \cdot \lVert e_{n-i}\rVert. \tag{submultiplicativity}
\end{align*}
Assuming that the local error is bounded, that is there exists $\lVert e_{max}\rVert$ such that $\lVert e_{i}\rVert \le \lVert e_{max}\rVert~\forall~i \in [1,n]$, we can simplify the upper bound to the global error to
\begin{align}\label{eq:: global_error_bound}
\lVert E_{n} \rVert \le& \lVert e_{max}\rVert \cdot \sum_{i=0}^{n-1} \lVert\tilde{\mathcal{K}}_d\rVert^i.
\end{align}
Note that
\begin{align*}
\lVert e_{max} \rVert = 0 \iff \lVert E_{k+n} \rVert = 0, 
\end{align*}
which is true for invariant subspaces.

From \eqref{eq:: global_error_bound}, if $\tilde{\mathcal{K}}_d$ is unstable, then the power of the matrix norm diverges as the number of time steps increases. This means that an unstable $\tilde{\mathcal{K}}_d$ amplifies exponentially even small errors, and thus renders the long-term prediction of the Koopman representation impractical. \textit{Instead, we propose that one should use a stable operator that generates similar local errors, but which also has the additional benefit of numerical stability over long-term predictions.} Note that it is possible that such a stable operator can generate similar local errors even for unstable dynamics. In fact, later in Section \ref{sec:: Results}, we show an example where a stable operator improves stabilization of a quadrotor that is unstable. 

Note that the eigenvalue profile of the Koopman model can be at times misleading in terms of bounding the power of the operator, because the upper bound can be itself large, as is pointed out in \cite{bounds_on_powers_of_matrices, peaks_discrete_linear,peak_bounds_discrete_linear}. Such scenarios are shown to occur for moderately large dimensions of the operator ($W \approx 100$) and do not apply in the examples of this paper. Exploiting the conditions that prevent large error growth in the transient response of a system, which is related to the strong stability property \cite{strong_stability_discrete}, for high-dimensional operators is left for future work. 

\subsection{Stability-Based Conditions for Koopman Basis Functions}
Given the importance of the spectral properties of the Koopman operator on the accuracy of long-term predictions, we next consider the implications of stability on the choice of basis functions. Specifically, we relate the stability properties of original nonlinear dynamics to the stability properties of the Koopman representation and present necessary conditions for the basis functions so that they are consistent with dynamics that have a stable (or an asymptotically stable) equilibrium. Although in practice one can impose stability on arbitrary Koopman models, choosing basis functions that are not consistent with stable dynamics could lead to unstable representations and worse training errors. We present the analysis for continuous-time dynamics, which are arguably the default expression, but equivalent relationships can be extended to the discrete-time case.

Consider a nonlinear dynamical system with states $s(t) \in \mathcal{S}$. We use $\mathrm{f}(s(t)) \in \mathbb{R}^{N} \mapsto \mathbb{R}^N$ to refer to the dynamics of the system represented in the state space $\mathcal{S}$, that is 
\begin{align}\label{eq:: nonlinear_dynamics_states}
 \mathrm{f}(s(t))\:\triangleq\:\frac{d}{dt} s(t) 
\end{align}
and $f(\Psi(s(t)))$ to refer to the associated Koopman dynamics given by
\begin{align} \label{eq:: nonlinear_dynamics_Koopman}
 f(\Psi(s(t)))\:\triangleq \frac{d}{dt} \Psi(s(t)) = \mathcal{K} \Psi(s(t)).
\end{align}

Then, we formally define the relationship of equivalent nonlinear and Koopman dynamics as follows.
\begin{definition}
Consider dynamics of the form in \eqref{eq:: nonlinear_dynamics_states} such that 
\begin{align}
 s(t) = s(0) + \int_0^t \mathrm{f}(s(\tau)) d\tau
\end{align}
and Koopman dynamics of the form in \eqref{eq:: nonlinear_dynamics_Koopman} such that
\begin{align}
 \Psi(s(t)) = \Psi(s(0)) + \int_0^t f(\Psi(s(\tau))) d\tau.
\end{align}
Then, we define dynamics \eqref{eq:: nonlinear_dynamics_states} and \eqref{eq:: nonlinear_dynamics_Koopman} to be equivalent if and only if
\begin{align}
 \int_0^t f(\Psi(s(\tau))) = \Psi(s(0)) + \Psi(\int_0^t \mathrm{f}(s(\tau)) d\tau).
\end{align}
Similarly, in discrete-time, 
\begin{align}
 f(\Psi(s(t_k))) = \Psi(s(t_k + \Delta t)) = \Psi(\mathrm{f}(s(t_k))).
\end{align}
\end{definition}

If the two representations \eqref{eq:: nonlinear_dynamics_states} and \eqref{eq:: nonlinear_dynamics_Koopman} are equivalent throughout the state space, that is, they evolve the dynamics in identically the same way, then we argue that certain properties must be true. 

\begin{proposition}\label{prop: conditions_for_Koopman_basis}
For a nonlinear dynamical system \eqref{eq:: nonlinear_dynamics_states} and an equivalent Koopman representation \eqref{eq:: nonlinear_dynamics_Koopman}, the following are true:
\begin{enumerate}
 \item If $s_e$ is an equilibrium for the nonlinear dynamical system \eqref{eq:: nonlinear_dynamics_states}, then it is an equilibrium for the equivalent Koopman dynamics \eqref{eq:: nonlinear_dynamics_Koopman}. That is, 
 \begin{align*}
 \mathrm{f}(s_e) = 0 \Longrightarrow f(\Psi(s_e)) = 0. 
 \end{align*}
 \item If $s_e$ is a Lyapunov-stable equilibrium for the equivalent nonlinear dynamical system \eqref{eq:: nonlinear_dynamics_states}, then it is a Lyapunov-stable equilibrium for the Koopman dynamics \eqref{eq:: nonlinear_dynamics_Koopman}. That is, if for every $\epsilon_s > 0$ there exists $\delta_s$ such that 
\begin{gather*}
 \| s(0) - s_e \| < \delta_s \Longrightarrow \| s(t) - s_e \| < \epsilon_s~
 \\\forall~t\ge 0,
\end{gather*}
then, for every $\epsilon_\Psi > 0$ there exists $\delta_\Psi$ such that
\begin{flalign*}
 &&\| \Psi(s(0)) - \Psi(s_e) \| <& \delta_\Psi& 
 \\\Longrightarrow && \| \Psi(s(t)) - \Psi(s_e) \| <& \epsilon_\Psi~\forall~t\ge 0.
\end{flalign*}
 
 \item If $s_e$ is an asymptotically stable equilibrium for the nonlinear dynamical system \eqref{eq:: nonlinear_dynamics_states}, then it is also an asymptotically stable equilibrium in $\mathcal{D}$ for the Koopman dynamics \eqref{eq:: nonlinear_dynamics_Koopman}. That is, if $s_e$ is Lyapunov stable for the nonlinear dynamical system \eqref{eq:: nonlinear_dynamics_states} and there exists $\delta_s$ such that, 
 \begin{align*}
 \|s(0) - s_e \| < \delta \Longrightarrow \lim_{t\to\infty} \|s(t) - s_e\| = 0,
 \end{align*}
 then, there exists $\delta_\Psi$
 \begin{align*}
 &\|\Psi(s(0)) - \Psi(s_e) \| < \delta_\Psi
 \\\Longrightarrow &\lim_{t\to\infty} \|\Psi(s(t)) - \Psi(s_e)\| = 0.
 \end{align*}
\end{enumerate}

\end{proposition}
\begin{proof}
See Appendix \ref{appendix: Necessary conditions for Koopman Basis Functions}. 
\end{proof}

Note that the conditions on the Koopman basis functions in Proposition \ref{prop: conditions_for_Koopman_basis} consider equivalent Koopman representations and nonlinear dynamics. These conditions need to be satisfied even for approximate Koopman operators in order for the Koopman representation to be consistent with the stability properties of the original system. Next, we use Proposition \ref{prop: conditions_for_Koopman_basis} to relate the stability properties of the original nonlinear dynamics to the properties of a Koopman operator and its associated basis functions. Using this relationship, we derive necessary conditions for both the operator and admissible basis functions associated with stable nonlinear systems. 

\subsubsection{Conditions on Koopman Operators for Stable Systems}
First, we derive stability properties for a Koopman representation that is consistent with its associated original stable, in the sense of Lyapunov, system. The analysis rests on Assumptions \ref{ass:: bounded_states} and \ref{ass:: psi_are_lipschitz}. 
\begin{assumption}\label{ass:: bounded_states}
All states $s(t) \in \mathbb{R}^N$ of a nonlinear dynamical system \eqref{eq:: nonlinear_dynamics_states} remain bounded for all $t$. That is, for some $\epsilon \ge 0$, $\lVert s(t) \rVert \le \epsilon$ for all $t \ge 0$. 
\end{assumption}

\begin{assumption}\label{ass:: psi_are_lipschitz}
The basis functions $\Psi(s(t))$ are Lipschitz with a Lipschitz constant $L_\Psi$. That is, 
\begin{align*}
 \| \Psi(s_1(t)) - \Psi(s_2(t)) \| \le L_\Psi \|s_1(t) - s_2(t)\| ~\forall~s_1, s_2 \in \mathcal{S}.
\end{align*}
\end{assumption}
\noindent Note that the Lipschitz constant $L_\Psi$ need not be known.

Next, we prove that only a stable Koopman operator can accurately represent nonlinear dynamics that are Lyapunov stable. 

\begin{definition}\label{def:: Bounded_Region}
Consider a nonlinear dynamical system \eqref{eq:: nonlinear_dynamics_states} with bounded states $s(t) \in \mathbb{R}^N$. We define $\mathcal{D}_\epsilon \subseteq \mathcal{S}$ as a region of the state space such that, if $\lVert s(t) \rVert \le \epsilon$, then $ s(t) \in \mathcal{D}_\epsilon$ for all $t \ge 0$. 
\end{definition}

\begin{theorem}[Lyapunov Stability]\label{th:: Luapunov_Stability}
Consider a nonlinear dynamical system \eqref{eq:: nonlinear_dynamics_states} and an equivalent Koopman representation \eqref{eq:: nonlinear_dynamics_Koopman}. If $s_e$ is a Lyapunov-stable equilibrium, then the Koopman operator $\mathcal{K}$ is not unstable. 
\end{theorem}
\begin{proof}
From Proposition \ref{prop: conditions_for_Koopman_basis}, if $s_e$ is a Lyapunov-stable equilibrium for the nonlinear dynamics, then it is also a Lyapunov-stable solution for the Koopman dynamics \eqref{eq:: nonlinear_dynamics_Koopman}.

\noindent From the definition of stability for linear systems \cite{antsaklis2006linear} (Definition 4.1 and Theorem 5.6), \cite{pritchard1981stability}, $\Psi(s_e)$ is a Lyapunov-stable equilibrium for the Koopman dynamics \eqref{eq:: nonlinear_dynamics_Koopman} if and only if $\mathcal{K}$ is not unstable, specifically $\mathrm{Re}[\lambda_i] \le0$, for each eigenvalue $\lambda_i$ of the operator $\mathcal{K}$ and each eigenvalue $\lambda_i$ with negative real part has an associated Jordan block of order one. Then, if $s_e$ is a Lyapunov-stable equilibrium for the nonlinear dynamics \eqref{eq:: nonlinear_dynamics_states}, then the Koopman operator $\mathcal{K}$ is not unstable.

\end{proof}

Theorem \ref{th:: Luapunov_Stability} proves that, for Koopman dynamics that model a system with a Lyapunov-stable equilibrium, it is a necessary condition that the Koopman operator is not unstable. Theorem \ref{th:: Luapunov_Stability} should also apply to approximate, data-driven Koopman operators. Consider a part of the state space $\mathcal{D}_{\epsilon}$ that contains a Lyapunov-stable equilibrium. Then a Koopman operator trained with measurements from $\mathcal{D}_{\epsilon}$ should be stable (see Theorem \ref{th:: Luapunov_Stability}) if it is to capture the original dynamics with full fidelity. Otherwise, if the model is unstable, states evolved with Koopman dynamics will diverge regardless of how close they start to the Lyapunov-stable equilibrium and, in this way, violate the second statement of Proposition \ref{prop: conditions_for_Koopman_basis}.

\subsubsection{Conditions on Koopman Operators for Asymptotically Stable Systems}

Next, we derive stability conditions for a Koopman representation that is consistent with its associated \textit{asymptotically} stable system. The analysis rests on Assumption \ref{ass:: One Asymptotic Equilibrium} and further assumes that all states of a nonlinear dynamical system \eqref{eq:: nonlinear_dynamics_states} lie inside a domain of attraction $D_0 \subseteq \mathbb{R}^N$, formally defined in Definition \ref{def:: Domain of Attracion}.
\begin{assumption}\label{ass:: One Asymptotic Equilibrium}
The nonlinear dynamical system \eqref{eq:: nonlinear_dynamics_states} has a single asymptotic equilibrium. 
\end{assumption}
\begin{definition}\label{def:: Domain of Attracion}
Given a nonlinear dynamical system \eqref{eq:: nonlinear_dynamics_states} and an asymptotically stable solution $s_e$, the domain of attraction is
\begin{align*}
 \mathcal{D}_0 \triangleq \{s_0 \in \mathcal{D} : \text{\textup{if}}~s(0) = s_0,~\text{\textup{then}}~\lim_{t\to\infty} s(t) = s_e\}. 
\end{align*}
\end{definition}

Note that for multiple asymptotic equilibria, there need to be separate regions of attractions, which must in turn be represented by separate Koopman operators. In this work, we focus on obtaining a single Koopman operator consistent with a single asymptotic equilibrium. For many systems, due to the presence of friction, this can often be represented as the zero-velocity state. For nonlinear systems with multiple equilibria points, one can use work in \cite{Tommy_DSS} to obtain multiple local Koopman representations.

Next, we prove that only a Koopman operator that is Hurwitz (all eigenvalues are strictly in the left half-plane) satisfies the stability properties of nonlinear dynamics that are asymptotically stable. 

\begin{theorem}[Asymptotic Stability]\label{th:: Asymptotic Stability}
Consider a nonlinear dynamical system \eqref{eq:: nonlinear_dynamics_states} and an equivalent Koopman representation \eqref{eq:: nonlinear_dynamics_Koopman}. Then, if $s_e$ is an asymptotically stable equilibrium, the Koopman operator $\mathcal{K}$ is Hurwitz.
\end{theorem}
\begin{proof}
From Proposition \ref{prop: conditions_for_Koopman_basis}, if $s_e$ is an asymptotically stable solution for the nonlinear dynamical system, it is also an asymptotically stable solution for the Koopman dynamics \eqref{eq:: nonlinear_dynamics_Koopman}. From the stability properties of linear systems, $\Psi(s_e)$ is an asymptotically stable equilibrium for the Koopman dynamics \eqref{eq:: nonlinear_dynamics_Koopman} if and only if $\mathcal{K}$ is Hurwitz: $\mathrm{Re}[\lambda_i(\mathcal{K})] <0$
for each eigenvalue $\lambda_i$ of the operator $\mathcal{K}$.
Then, if $s_e$ is an asymptotically stable equilibrium for the nonlinear dynamics, then the Koopman operator $\mathcal{K}$ of the associated Koopman dynamics is Hurwitz.
\end{proof}
 
Proposition \ref{prop: conditions_for_Koopman_basis} and Theorems \ref{th:: Luapunov_Stability} and \ref{th:: Asymptotic Stability} present conditions for a Koopman representation, specifically on admissible basis functions and the operator itself, that is consistent with the stability properties of the associated underlying nonlinear system. In practice, one can compute a Koopman operator for any choice of basis functions. However, the stability-related constraints on the admissible basis functions can indicate which basis functions are consistent with stable Koopman operators. We argue that using basis functions that violate these conditions and are inconsistent with stable dynamics would generate unstable Koopman operator solutions or, if stability is enforced on the operator during learning as is the focus of this work, would lead to higher training error. As we show next, this becomes problematic, as bounding the training error is of central importance in the construction of Lyapunov functions.
\subsection{Lyapunov Functions Using Stable Koopman Operators}
Given a stable operator, it is possible to design Lyapunov functions for nonlinear systems using the data-driven Koopman matrix. Consider an approximate finite-dimensional Koopman representation that is equivalent, based on Proposition \ref{prop: conditions_for_Koopman_basis}, to general dynamical systems \eqref{eq:: nonlinear_dynamics_states} such that
\begin{align}\label{eq:: nonlinear_dynamics_remainder}
 \frac{d}{dt} \Psi (s(t))= \tilde{\mathcal{K}} \Psi(s(t)) + \epsilon(\Psi(s(t))),
\end{align}
where 
\begin{align}
 \epsilon(\Psi(s(t))) \triangleq f(\Psi(s(t)) - \tilde{\mathcal{K}} \Psi(s(t))
\end{align} is the residual error. 
If $\tilde{\mathcal{K}}$ in \eqref{eq:: nonlinear_dynamics_remainder} is stable, we can design Lyapunov functions for the nonlinear dynamics, as we prove in Theorem \ref{th:: SemiGlobalStability}. 

\begin{theorem}\label{th:: SemiGlobalStability}
Consider a Koopman representation \eqref{eq:: nonlinear_dynamics_remainder} of a nonlinear dynamical system. Further, assume $\tilde{\mathcal{K}}$ is stable. Let 
\begin{align}\label{eq:: a_max}
 \alpha \le 
 \dfrac{\lambda_{min}(Q_{\tilde{\mathcal{K} }})}{ 2 \lambda_{max}(P)}
\end{align}
where $Q \succ 0$ and $P \succ 0$ are a solution to the Lyapunov equation
\begin{equation}\label{eq:: LyapunovEquation}
 \tilde{\mathcal{K}}^TP + P \tilde{\mathcal{K}} + Q_{\tilde{\mathcal{K}} } = 0.
\end{equation}
If
\begin{gather}
 \lVert \epsilon(\Psi(0)) \rVert_2 = 0 \label{eq:: zeroError} \intertext{and} \lVert \epsilon(\Psi(s(t))) \rVert_2 \le \alpha \lVert \Psi (s(t)) \rVert_2 ~\forall~\Psi(s(t)) \in \mathcal{D}_{\alpha} \subseteq \mathcal{S},\label{eq:: boundedError}
\end{gather}
then, the zero solution $\Psi(s(t)) = 0$ to the nonlinear dynamics is asymptotically stable in $\mathcal{D}_{\alpha}$. Further, $V = \Psi(s(t))^T P \Psi(s(t))$ is a Lyapunov function in $\mathcal{D}_{\alpha}$. 
\end{theorem}
\begin{proof} 
Consider a candidate Lyapunov function $V = \Psi^T(s(t))P \Psi(s(t))$. Taking the time derivative, 
\begin{align*}
 \frac{d}{dt} V (\Psi(s(t))) =& \frac{dV(\Psi(s(t))}{d\Psi(s(t))} \frac{d\Psi(s(t))}{dt} \\
 =& \Psi(s(t))^T P f(\Psi(s(t))) \\
 &+ f(\Psi(s(t)))^T P \Psi(s(t)) \\
 =& \Psi(s(t))^T P [\mathcal{K} \Psi(s(t)) + \epsilon(\Psi(s(t)))] 
 \\&+ [ \Psi^T \tilde{\mathcal{K}}^T + \epsilon(\Psi(s(t)))^T] P \Psi(s(t)) \\
 =& \Psi(s(t))^T (P \tilde{\mathcal{K}} + \tilde{\mathcal{K}}^T P) \Psi(s(t)) 
 \\&+ 2 \Psi(s(t))^T P \epsilon(\Psi(s(t))).
\end{align*}

Given that $\tilde{\mathcal{K}}$ is stable and $Q_{\tilde{\mathcal{K}} } \succ 0$, then there always exists a (unique) solution $P \succ 0$ to the Lyapunov equation.
Thus,
\begin{align*}
 P \tilde{\mathcal{K}} + \tilde{\mathcal{K}}^T P = - Q_{\tilde{\mathcal{K}} }
\end{align*}
such that 
\begin{align*}
 \frac{d}{dt} V (\Psi(s(t))) =& -\Psi(s(t))^T Q_{\tilde{\mathcal{K}} } \Psi(s(t)) \\ &+ 2 \Psi(s(t))^T P \epsilon(\Psi(s(t))).
\end{align*}

Note that $ - \Psi(s(t))^T Q_{\tilde{\mathcal{K}} } \Psi(s(t)) \le - \lambda_{min}(Q_{\tilde{\mathcal{K}}}) \lVert \Psi(s(t)) \rVert_2^2$ and using the Cauchy-Schwartz inequality \cite{haddad2011nonlinear}, it follows that
\begin{align*}
 \frac{d}{dt} V (\Psi(s(t))) \le& - \lambda_{min}(Q_{\tilde{\mathcal{K}} }) \lVert \Psi(s(t)) \rVert_2^2 
 \\&+ 2 \lambda_{max}(P) \lVert \Psi(s(t)) \rVert_2 \lVert \epsilon(\Psi(s(t)))\rVert_2.
\end{align*}
Then, using \eqref{eq:: boundedError}, we can rewrite
\begin{gather*}
\frac{d}{dt} V (\Psi(s(t))) \le - (\lambda_{min}(Q_{\tilde{\mathcal{K}} } ) - 2\alpha\lambda_{max}(P))\lVert )\lVert \Psi(s(t)) \rVert_2^2\\
~\forall~\Psi(s(t)) \in \mathcal{D}_{\alpha} \subseteq \mathcal{S}.
\end{gather*}
Choosing 
$\alpha <= \lambda_{min}(Q_{\tilde{\mathcal{K}} })/ 2\lambda_{max}(P)$, then
\begin{align*}
\dfrac{d}{dt}V(\Psi(s(t))) \le 0 ~\forall~\Psi(s(t)) \in \mathcal{D}_{\alpha}
\end{align*}
such that the system \eqref{eq:: nonlinear_dynamics_remainder} is stable in $\mathcal{D}_{\alpha}$ about $\Psi(s(t))=0$. Further, for $\alpha < \lambda_{min}(Q_{\tilde{\mathcal{K}} })/ 2\lambda_{max}(P)$, $\dfrac{d}{dt}V(\Psi(s(t))) < 0 \in \mathcal{D}_{\alpha}$ and the system is asymptotically stable in $\mathcal{D}_{\alpha}$ about $\Psi(s(t)) = 0$.
\end{proof}

Theorem \ref{th:: SemiGlobalStability} makes it is possible to design Lyapunov functions for nonlinear systems that are valid in a region $\mathcal{D}_{\alpha}$ of the state space where the modeling error of the Koopman operator is bounded \eqref{eq:: boundedError}. This highlights the importance of choosing appropriate basis functions that satisfy the stability conditions outlined in Proposition \ref{prop: conditions_for_Koopman_basis} and the conditions in \eqref{eq:: zeroError} and \eqref{eq:: boundedError}. Further, Theorem \ref{th:: SemiGlobalStability} makes it possible to design control-Lyapunov functions that can be used to verify stabilizing feedback laws from controlled measurements in a region $\mathcal{D}_{\alpha}$ \eqref{eq:: boundedError} for data-driven systems. In Section \ref{sec:: Results}, we present a few examples of verifying the stability of controlled systems using Lyapunov functions constructed from stable data-driven Koopman operators, but leave more sophisticated analysis for future work. 

Section \ref{sec:: Stability_properties_of_Koopman} shows that stable Koopman operators can improve long-term predictions and help construct Lyapunov functions. We also present conditions on admissible basis functions for stable Koopman models to match the stability properties of the operator and improve the training error. However, in practice, data-driven Koopman operators can be unstable even if the modeled dynamics are stable and even if appropriate basis functions (as described in Proposition \ref{prop: conditions_for_Koopman_basis}) are used. For this reason, next in Section \ref{sec::DISKO} we present the first framework for data-driven identification of stable Koopman operators (DISKO). 

\section{Synthesis of Stable Koopman Operators}\label{sec::DISKO}
This section presents the methodology used for the Data-driven Identification of Stable Koopman Operators (DISKO). There are several candidate algorithms that can compute a stable solution to the optimization in \eqref{eq:: LS_StableKoopman}. In this paper, we use the gradient-descent algorithm presented in \cite{mamakoukas_stableLDS2020} (referred to as SOC) to find locally optimal stable solutions $\tilde{\mathcal{K}}_d$. The SOC algorithm builds upon the work in \cite{NearestStable, gillis2020note}, where the nearest stable matrix to an unstable one is computed by minimizing the Frobenius norm. 
 
We choose the SOC algorithm because it is shown to outperform the top-alternative existing algorithms for stable linear dynamical systems in terms of model accuracy, memory efficiency, and scalability such that it can be used for feedback control of higher-dimensional systems when the other methods fail \cite{mamakoukas_stableLDS2020}. 
\subsection{Stable Least-Squares Koopman Operators} 
To compute a stable Koopman operator, we first convert the optimization \eqref{eq:: LS_StableKoopman} to a different formulation that is suitable for the SOC algorithm.

\begin{proposition}\label{prop:: EDMD_Koopman}
Consider $P$ measurements of states $s \in \mathbb{R}^N$ and basis functions $\Psi(s(t)) \in \mathbb{R}^W$. Given $X$ and $Y$ such that
\begin{align*}
X =& \begin{bmatrix} \Psi(s(t_1), u (t_1 ))^T \\ \vdots \\ \Psi(s(t_P), u (t_P))^T
 \end{bmatrix}^T
 \intertext{and}
 Y =& \begin{bmatrix} \Psi(s(t_1 + \Delta t), u (t_1 + \Delta t))^T \\ \vdots \\ \Psi(s(t_P + \Delta t), u (t_P + \Delta t))^T
 \end{bmatrix}^T.
\end{align*}
Then, the expression 
\begin{align*}
\sum_{k = 1}^{P}\frac{1}{2}\lVert \Psi(s(t_k + \Delta t), u(t_k + \Delta t)) - \tilde{\mathcal{K}}_d \Psi(s(t_k), u(t_k))\rVert^2
\end{align*}
is equivalent to 
\begin{align*}
\frac{1}{2}\lVert Y - \tilde{\mathcal{K}}_d X \rVert_{F}^2,
\end{align*}
where $X, Y \in \mathbb{R}^{W \times P}$, $\tilde{\mathcal{K}}_d \in \mathbb{R}^ {W \times W}$, $\lVert \cdot \rVert_F$ is the Frobenius norm of a matrix and $\mathbb{S}^{W,W}_d$ is the set of all stable matrices of size $W \times W$.
\end{proposition}
\begin{proof}
See Appendix \ref{App:: Equivalence}. 
\end{proof}

\begin{algorithm*}[ht!]
 \textbf{Input:} $\Psi_s(t), \Psi_u(t)$ \Comment{Choice of basis functions}\\
 \textbf{Output:} \text{Stable Koopman model}
 \caption{Data-driven Identification of Stable Koopman Operators (DISKO)}\label{algo:: DISKO}
\begin{algorithmic}[1]
\While{$k < k_{max}$}
 \LState \text{Collect new state and control measurements} $s(t_k), s(t_k + \Delta t)$,  $u(t_k)$ 
 \LState \text{Evaluate basis functions $\Psi_s(s(t_k)), \Psi_s(s(t_k + \Delta t)), \Psi_u(t_k)$} 
 \LState \text{Update $\mathcal{G}, \mathcal{A}, X_U, Y_U, U_U$} \Comment{Preserve memory space}
 \LState \text{Run SOC algorithm} \Comment{Compute stable Koopman dynamics \eqref{eq:: AxBu_Koopman}}
 \EndWhile
 \end{algorithmic}
\end{algorithm*}

From Proposition \ref{prop:: EDMD_Koopman}, seeking stable Koopman operators for
\begin{equation}\label{eq:: LS_StableKoopman}
\inf \limits_{\tilde{\mathcal{K}}_d \in \mathbb{S}^{W,W}_d} \frac{1}{2}\lVert Y - \tilde{\mathcal{K}}_d X \rVert_{F}^2,
\end{equation}
is equivalent to seeking stable solutions for \eqref{eq:: Kd_LS}. Note that solving \eqref{eq:: LS_StableKoopman} is not equivalent to projecting the unconstrained Koopman solution \eqref{eq:: Kd_LS} to the stable set of matrices. That is, 
\begin{equation}\label{eq:: LS-Stable_vs_Stable}
\inf \limits_{\tilde{\mathcal{K}}_d \in \mathbb{S}^{W,W}_d} \frac{1}{2}\lVert Y - \tilde{\mathcal{K}}_d X \rVert_{F}^2 \neq \inf \limits_{\tilde{\mathcal{K}}_d \in \mathbb{S}^{W,W}_d} \frac{1}{2}\lVert \tilde{\mathcal{K}}^*_d - \tilde{\mathcal{K}}_d \rVert_{F}^2.
\end{equation}
Projecting an unstable solution of \eqref{eq:: Kd_LS} to the stable set results in a matrix that is stable but often with much greater fitness error than the solution to \eqref{eq:: LS_StableKoopman}, as we demonstrate with examples later in Section \ref{sec:: Results}. 

The SOC algorithm uses the property that a matrix $A$ is stable if and only if it can be written as $A = S^{-1} OCS$\cite{NearestStable}, where $S$ is invertible, $O$ is orthogonal, and $C$ is a positive semidefinite contraction; its singular values are less than or equal to 1. Let $\mathbb{I}^{W\times W}$ be the set of all invertible matrices of size $W \times W$. Then, we reformulate the optimization \eqref{eq:: LS_StableKoopman} such that
\begin{equation*}
\resizebox{0.99\hsize}{!}{$
\begin{aligned}
\inf \limits_{K_d \in \mathbb{S}^{W,W}_d} \frac{1}{2}\lVert Y - \tilde{\mathcal{K}}_d X \rVert_{F}^2 = \inf \limits_{S \in \mathbb{I}^{W \times W}, O \text{ orthogonal}, C\succeq 0, \lVert C\rVert \le 1 } \frac{1}{2}\lVert Y - S^{-1}OCSX\rVert_F^2.
 \end{aligned}
 $}
\end{equation*}

When considering systems with inputs, instead of imposing stability on the Koopman operator that propagates both state- and input- dependent basis functions, we use the Koopman dynamics shown in \eqref{eq:: AxBu_Koopman} and impose stability only the state transition matrix $A$. Then, the optimization problem becomes 
\begin{equation}\label{eq:: SOC_withInputs}
\resizebox{0.99\hsize}{!}{$
[A, B] = \inf \limits_{S\in \mathbb{I}^{W_s \times W_s}, O \text{ orthogonal}, C\succeq 0, \lVert C\rVert \le 1 } \frac{1}{2}\lVert Y - S^{-1}OCSX - BU\rVert_F^2,
 $}
\end{equation}
where $X, Y \in \mathbb{R}^{W_s \times P}$ include the measurements of the state-dependent Koopman basis functions $\Psi_s(s(t))$ and $U \in \mathbb{R}^{W_u \times P}$ include the measurements of the control-dependent ones $\Psi_u(u(t))$, similar to the form shown in proposition \ref{prop:: EDMD_Koopman}. Note that for linear systems with inputs, the stability properties described in Section \ref{sec:: Stability_properties_of_Koopman} apply to the state-transition matrix $A$ shown in \eqref{eq:: AxBu_Koopman}. In the rest of the analysis, one can simply set $B = 0$ for systems without inputs. For details on how the SOC algorithm can solve either \eqref{eq:: LS_StableKoopman} or \eqref{eq:: SOC_withInputs}, we refer the reader to \cite{mamakoukas_stableLDS2020} and the publicly available code\footnote{\href{https://github.com/MurpheyLab/MemoryEfficientStableLDS}{https://github.com/MurpheyLab/MemoryEfficientStableLDS}}.

Let $f(S,O,C, B) = \frac{1}{2}\lVert Y - S^{-1}OCSX - BU\rVert_F^2$. The gradients with respect to $S, O$, and $C$ are derived in \cite{mamakoukas_stableLDS2020} and rewritten here in a more compact form as 
\begin{equation}\label{eq:: Gradient Descent of S, O, C}
\begin{aligned}
\nabla_S f(S,O,C, B) =&~S^{-T} [\mathcal{V} A^T - A^T \mathcal{V}]\\
\nabla_O f(S,O,C, B) =& - S^{-T} \mathcal{V} S^T C^T\\
\nabla_C f(S,O,C, B) =& - O^T S^{-T}\mathcal{V} S^T \\
\nabla_B f(S,O,C, B) =& - (Y - A X - BU) U^T
\end{aligned}
\end{equation}
where $A = S^{-1}OCS \in \mathbb{R}^{W_s \times W_s}$ and $\mathcal{V} = (Y - A X - BU) X^T \in \mathbb{R}^{W_s \times W_s}$. The gradients \eqref{eq:: Gradient Descent of S, O, C} depend on $X$, $Y$, and $U$, which contain a history of all the basis functions measurements and can slow down the computation over time, as increasingly more data are collected. Compared to \cite{mamakoukas_stableLDS2020} and in order to  speed up the computation as well as preserve memory space, we use the relationships $XX^T = \mathcal{G} \in \mathbb{R}^{W_s \times W_s}$, $YX^T = \mathcal{A} \in \mathbb{R}^{W_s \times W_s}$, $XU^T = X_U \in \mathbb{R}^{W_s \times W_u}$, $YU^T = Y_U \in \mathbb{R}^{W_s \times W_u}$, and $UU^T = U_U \in \mathbb{R}^{W_u \times W_u}$ (derived in Appendix \ref{App:: Memory-Preserving Gradient Descents}) such that $\mathcal{V} = \mathcal{A} - A\mathcal{G} - BX_U^T$ and $\nabla_B f(S,O,C, B) = - Y_U + A X_U + BU_U$. Then, the gradient directions can be incrementally updated with new measurements and preserve memory space. We present the algorithmic steps of the DISKO framework in Algorithm \ref{algo:: DISKO}.

\section{Results}\label{sec:: Results}
In this section, we demonstrate the benefits of DISKO in prediction and control. First, we consider systems without inputs and show the effect of imposing stability on the prediction accuracy. We first compare the solutions of the proposed SOC algorithm to the method \cite{gillis2020note} that does not consider the least-squares fitness error when projecting the matrix to the stable set of solutions. Then, we compare the evolution of nonlinear dynamics using the unconstrained and stable Koopman models, shown in \eqref{eq:: Kd_LS} and \eqref{eq:: LS_StableKoopman}, respectively. 

\subsection{Least-Squares vs Nearest Stabilization}
We demonstrate the difference between projecting an unstable matrix to the nearest stable solution (nearest stabilization) and solving for a stable matrix while also optimizing for the reconstruction error (least-squares stabilization).  

\begin{figure*}
 \subcaptionbox{CG \label{fig::DISKO vs LDS}}{
 \includegraphics[width = 0.65\columnwidth, keepaspectratio = true]{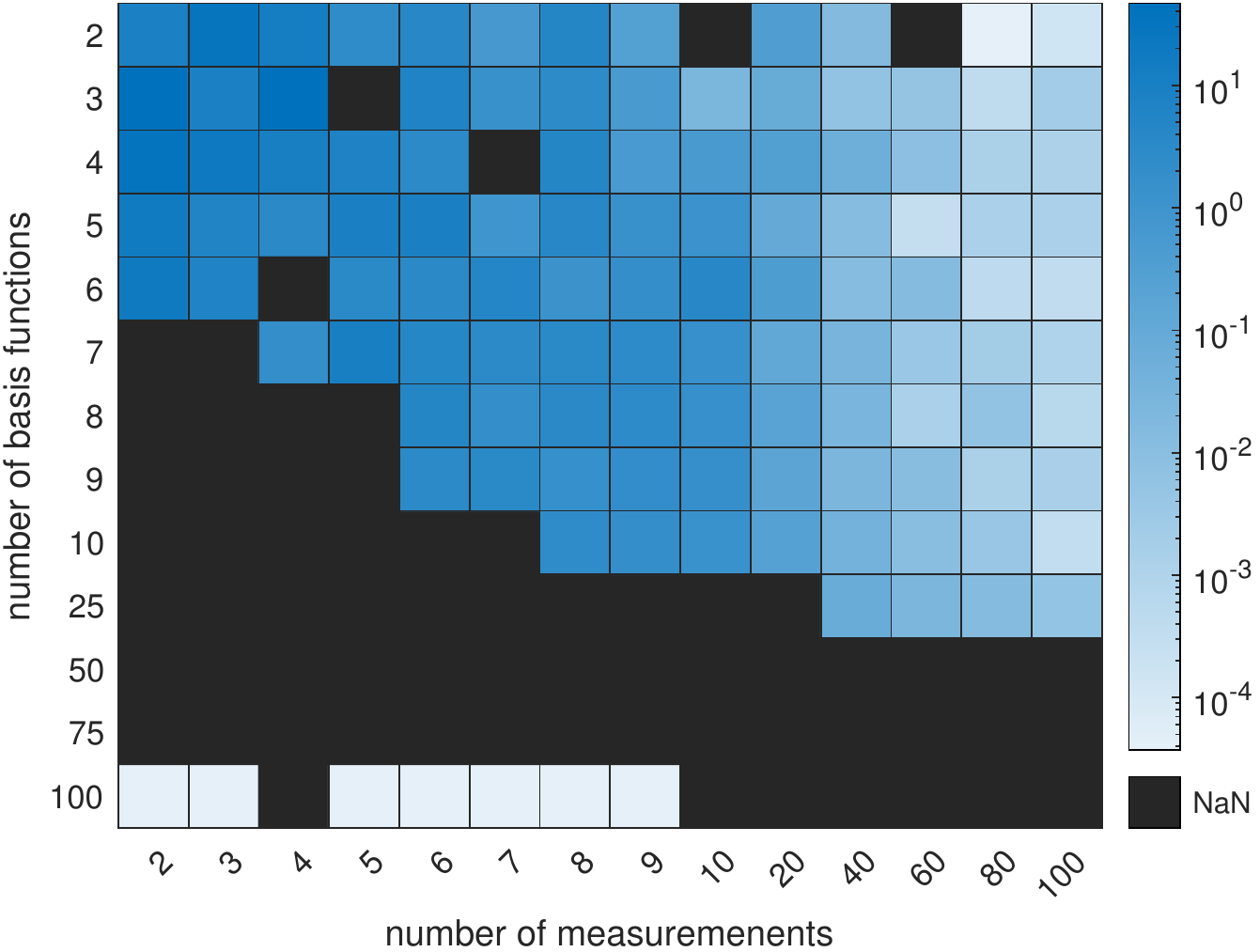}}
 \hfill
 \subcaptionbox{DISKO}{
\includegraphics[width = 0.65\columnwidth, keepaspectratio = true]{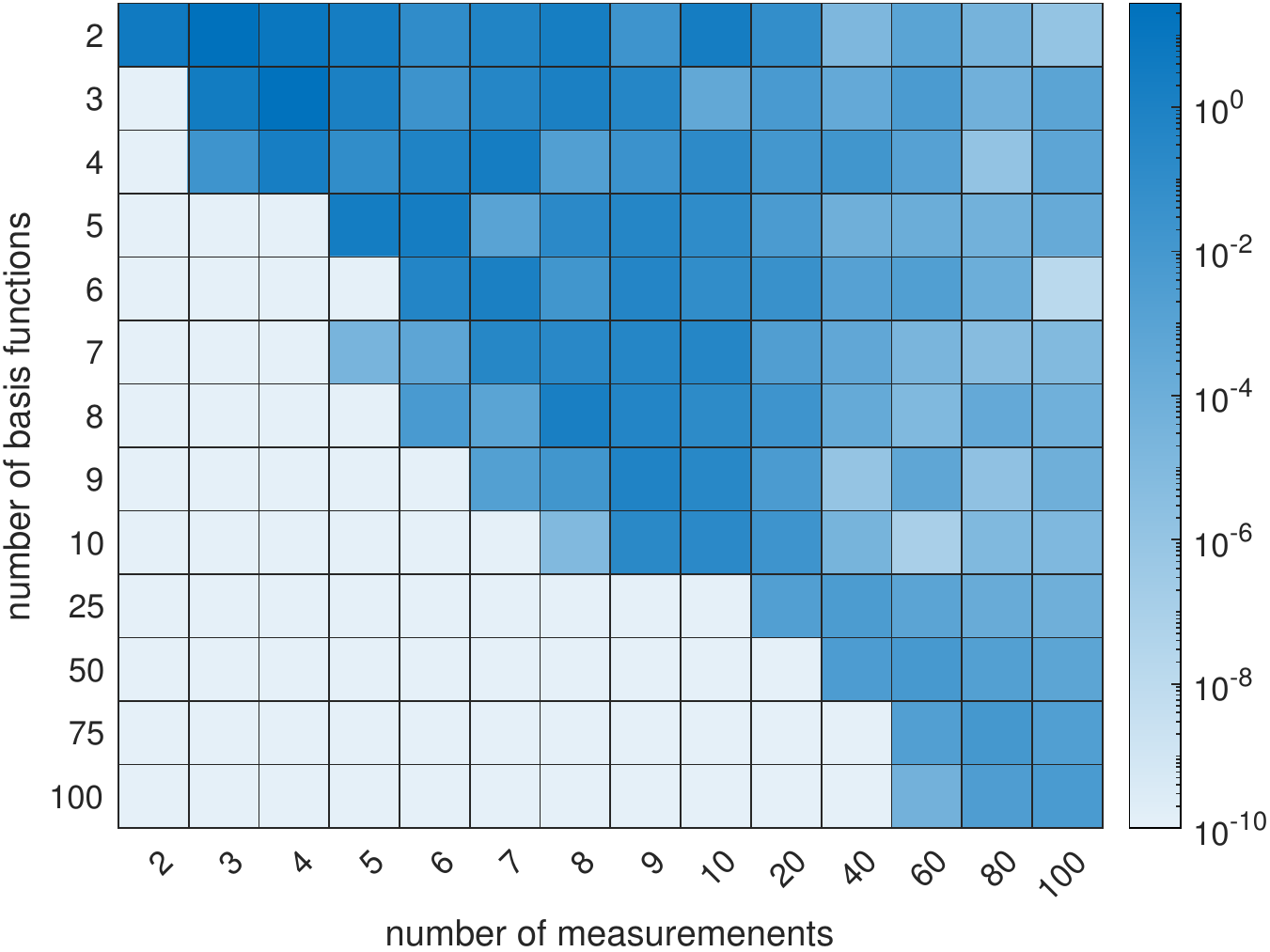}}
 \hfill
 \subcaptionbox{Percent Difference \label{Fig:: difference}}{
\includegraphics[width = 0.65\columnwidth, keepaspectratio = true]{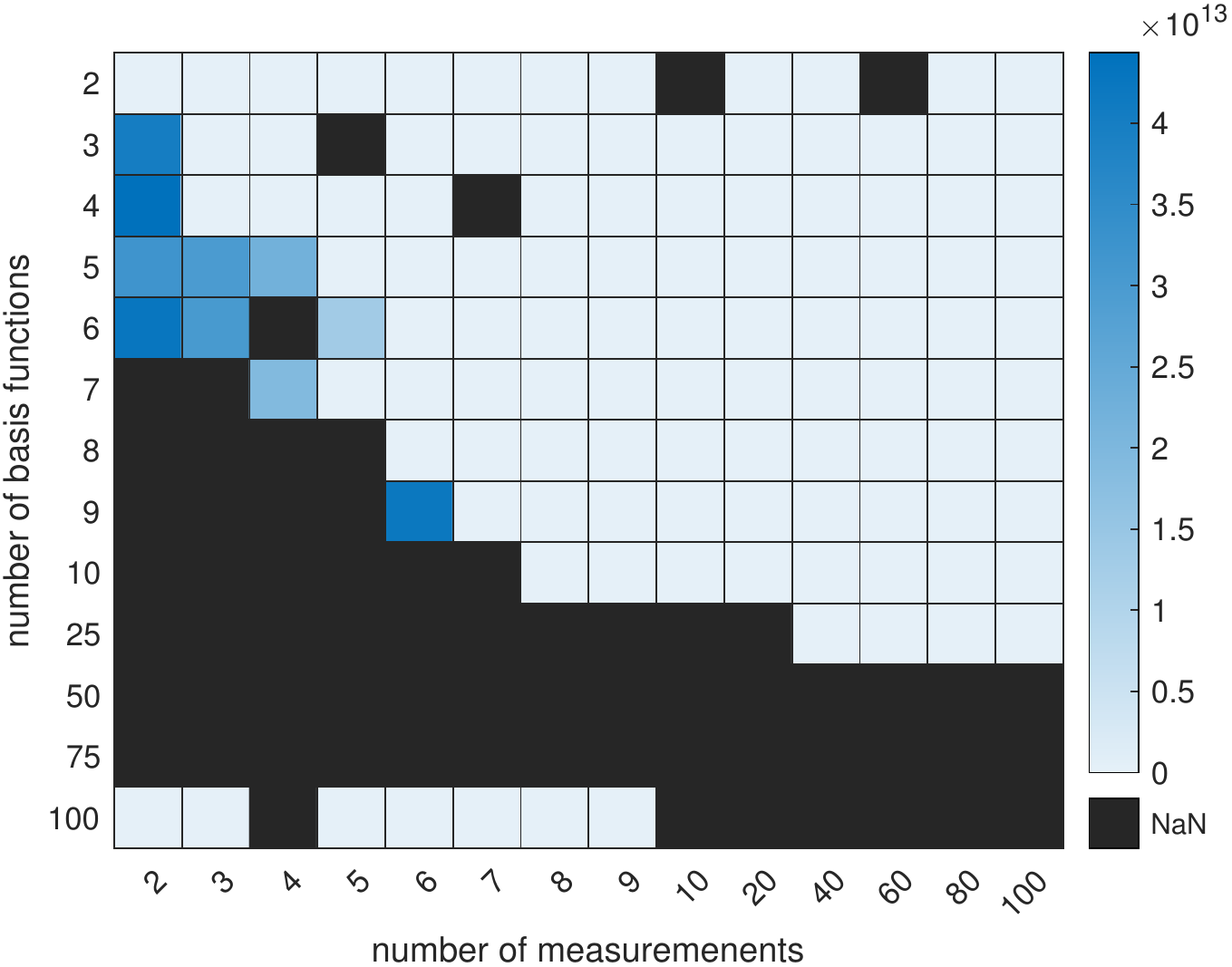}}

 \caption{Comparison of the SOC (Algorithm \ref{algo:: DISKO}) and CG \cite{boots2008constraint} algorithms as a function of the total number of random measurements used for training and the total number of basis functions. The error is normalized by the product of the number of measurements and functions. In \ref{Fig:: difference}, we calculate the percent difference of the error between the two algorithms: $\frac{e_{CG} - e_{SOC}}{e_{SOC}}$. \label{Fig:: SOC vs CG}}
\end{figure*}

 Consider the randomly generated matrices
 \begin{align*}
 X =& \begin{bmatrix} 0.1419 & 0.4218 & 0.9157 & 0.7922 & 0.9595 \\
 0.6557 & 0.0357 & 0.8491 & 0.9340 & 0.6787 \\
 0.7577 & 0.7431 & 0.3922 & 0.6555 & 0.1712 \end{bmatrix} \\
 Y =& \begin{bmatrix} 8.1472 & 9.0579 & 1.2699 & 9.1338 & 6.3236\\
 0.9754 & 2.7850 & 5.4688 & 9.5751 & 9.6489 \\
 1.5761 & 9.7059 & 9.5717 & 4.8538 & 8.0028 \end{bmatrix}. 
 \end{align*}
 Computing the least-squares solution using \eqref{eq:: Kd_LS} and then projecting it to the stable set of matrices using \cite{gillis2020note} yields
 \begin{align*}
 \tilde{\mathcal{K}}_d = \begin{bmatrix} 0.0041 & -6.6031 & 5.1709 \\
 10.3449 & -1.9480 & -0.0590 \\
 11.7192 & -6.7149 & 3.4609 \end{bmatrix},
 \end{align*}
 with eigenvalues $\Lambda = \{0.87, 0.87, -0.22\}$. The least-squares error using the stable matrix is 
 \begin{align*}
 \frac{1}{2}\lVert Y - \tilde{\mathcal{K}}_d X \rVert_{F}^2 = 203.04
 \end{align*}
 and the Frobenius norm from the least squares solution is
 \begin{align*}
 \frac{1}{2}\lVert \tilde{\mathcal{K}}^*_d - \tilde{\mathcal{K}}_d \rVert_{F}^2 = 45. 98.
 \end{align*}
 On the other hand, directly solving \eqref{eq:: LS_StableKoopman} generates a different solution
 \begin{align*}
 \tilde{\mathcal{K}}_d = \begin{bmatrix} 5.6337 & -8.2334 & 11.5883 \\
 14.4877 & -5.0863 & 1.9636 \\
 8.3346 & -2.8916 & 1.0662\end{bmatrix}
 \end{align*}
 with eigenvalues $\Lambda = \{0.98, 0.98, -0.35\}$. The least-squares error using the stable matrix is 
 \begin{align*}
 \frac{1}{2} \lVert Y - \tilde{\mathcal{K}}_d X \rVert_{F}^2 = 79.47
 \end{align*}
 and the Frobenius norm from the least squares solution is
 \begin{align*}
 \frac{1}{2}\lVert \tilde{\mathcal{K}}^*_d - \tilde{\mathcal{K}}_d \rVert_{F}^2 =108.53.
 \end{align*}
 As expected, projecting the unstable solution to the stable set and ignoring the least-squares error fitness generates a solution that is closer, in the Frobenius norm sense, to the original unstable matrix, but also with greater error compared to the solution of \eqref{eq:: LS_StableKoopman}. 
 
\subsection{Comparison to alternative schemes for DISKO}
As mentioned earlier, there are many candidate algorithms that can be implemented for DISKO. To motivate using the SOC algorithm over alternative choices, we compare it to the constraint generation (CG) approach \cite{boots2008constraint}, which has been shown to outperform competing alternative algorithms.\footnote{The algorithm in \cite{WLS_stableLDS} does not work well for systems with inputs, as demonstrated in \cite{mamakoukas_stableLDS2020}.} The two algorithms are compared also in \cite{mamakoukas_stableLDS2020}, but not on randomly generated matrices and not up to such high dimensions. 

We compare SOC and CG on finding stable operators that minimize the error in 
\eqref{eq:: LS_StableKoopman} for varying number of basis functions ($W \in [2, 100]$ and number of measurements ($ P \in [2,100]$). Data in $X$ and $Y$ are sampled from the uniform distributions $U(0,10)$ and $U(0,20)$, respectively, where $U(a,b)$ is a uniform distribution and $a$ and $b$ are the minimum and maximum values. We present the results in Fig. \ref{Fig:: SOC vs CG}. SOC outperforms CG in all cases, for any number of measurements and functions used. Further, CG does not always converge to a solution in the allotted time (10 minutes per minimization). Besides the superior performance in terms of the reconstruction error, the SOC algorithm is also more memory efficient. Since SOC performs better than the alternative algorithms for randomized matrices, it is a good choice for learning stable Koopman models. In light of these results, we use the SOC algorithm to implement DISKO in the remaining of this work.

\subsection{Comparisons of Reconstruction and Prediction Error}
Next, we compare the evolution of the nonlinear dynamics of a pendulum (without control) using the unconstrained solution \eqref{eq:: Kd_LS} and the stable one \eqref{eq:: LS_StableKoopman}. The states, dynamics, and basis functions used to train the Koopman operator are given by 
\begin{gather*}
    s = [\theta, \dot\theta]^T, \qquad     \frac{d}{dt} s = [\dot\theta, 9.81 \sin(\theta) + \beta \dot\theta]^T \\
    \Psi(s) = [\theta, \dot\theta, \sin(\theta), \cos(\theta)  \dot\theta, \sin(\theta)\cos(\theta), \sin(\theta) \dot\theta^2]^T,
\end{gather*}
where $\theta, \dot\theta$, and $\beta$ are the angle, angular velocity and damping coefficient, respectively. The time spacing between samples is $\Delta t = 0.02$~s. For simplicity, we drop the time dependencies of the variables. 

The equilibrium points of the pendulum dynamics occur when $\theta = n\pi$ for $n \in \mathcal{Z}$ and $\dot\theta = 0$. The basis functions used here evaluate to zero when $\theta = 0$ and $\dot\theta = 0$, satisfying Proposition \ref{prop: conditions_for_Koopman_basis} for one of the equilibrium points of the original dynamics. Note that, depending on the identified Koopman model, it is still possible that the Koopman dynamics are zero (i.e. $\tilde{\mathcal{K}}\Psi(s(t)) = 0$) at the rest of the equilibrium points of the pendulum system. Omitting $\theta$ from the basis functions would ensure that the first condition in Proposition \ref{prop: conditions_for_Koopman_basis} is satisfied for all the equilibrium points of the pendulum. In that case, however, and using a single Koopman model, it would not be possible to differentiate between the unstable and stable equilibrium as target states for control purposes. For this reason, we include $\theta$ in the states. In future work, we plan on solving for Koopman models subject to the constraints that the Koopman dynamics share the same equilibrium points as the original nonlinear system. 

Fig. \ref{fig:: PendulumPrediction} shows the eigenvalues of the two methods and the prediction of the two solutions for the undamped pendulum ($\beta = 0$). We train the Koopman models with data that are collected by propagating the pendulum dynamics by $\Delta t = 0.02$~s from 1000 random initial conditions sampled uniformly from $U(-\pi, \pi)$ and $U(-1, 1)$ for the angle and the angular velocity, respectively. The prediction associated with the unconstrained least-squares Koopman operator diverges away, whereas the states predicted with DISKO remain bounded and close to the true evolution of the nonlinear dynamics for a longer amount of time.

In this example, enforcing stability around the unstable point may lead to poor prediction performance. However, we find that choosing to use a stable operator to model a locally unstable point bounds the instability as shown in Fig.~\ref{fig:: PendulumPrediction}. Numerically, this is advantageous, since the stable model prevents divergent predictions on long time horizons. On the other hand, unstable least-squares solutions generate exponentially diverging predictions that are impractical to use for model-based control. That said, depending on the stability properties of the underlying dynamics, it might be more beneficial to \textit{bound} the instability of data-driven models ---instead of forcing stability--- to prevent numerically challenging predictions and improve data-driven control. Or, alternatively, an adaptive algorithm could learn---using a single parameter---how much instability is needed based on data. The SOC algorithm used here can explicitly bound the magnitude of the eigenvalues for discrete-time models, see Section~\ref{ssec::global_error}.

\begin{figure}
 \subcaptionbox{\label{subfig: eigenvalues}}{
 \includegraphics[width = 0.88\columnwidth, keepaspectratio = true]{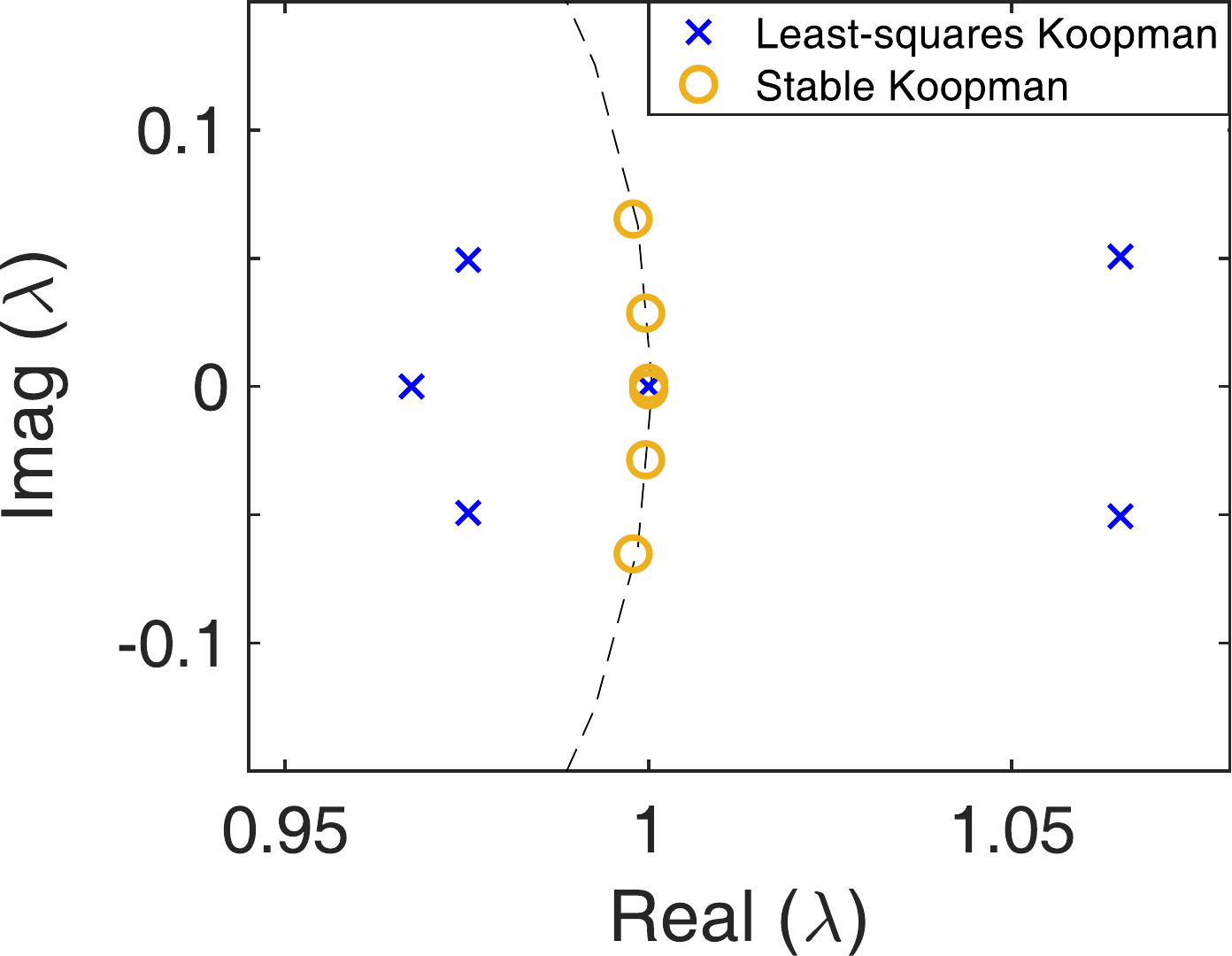}}
 \hfill
 \subcaptionbox{\label{subfig: prediction}}{
\includegraphics[width = 0.48\columnwidth, keepaspectratio = true]{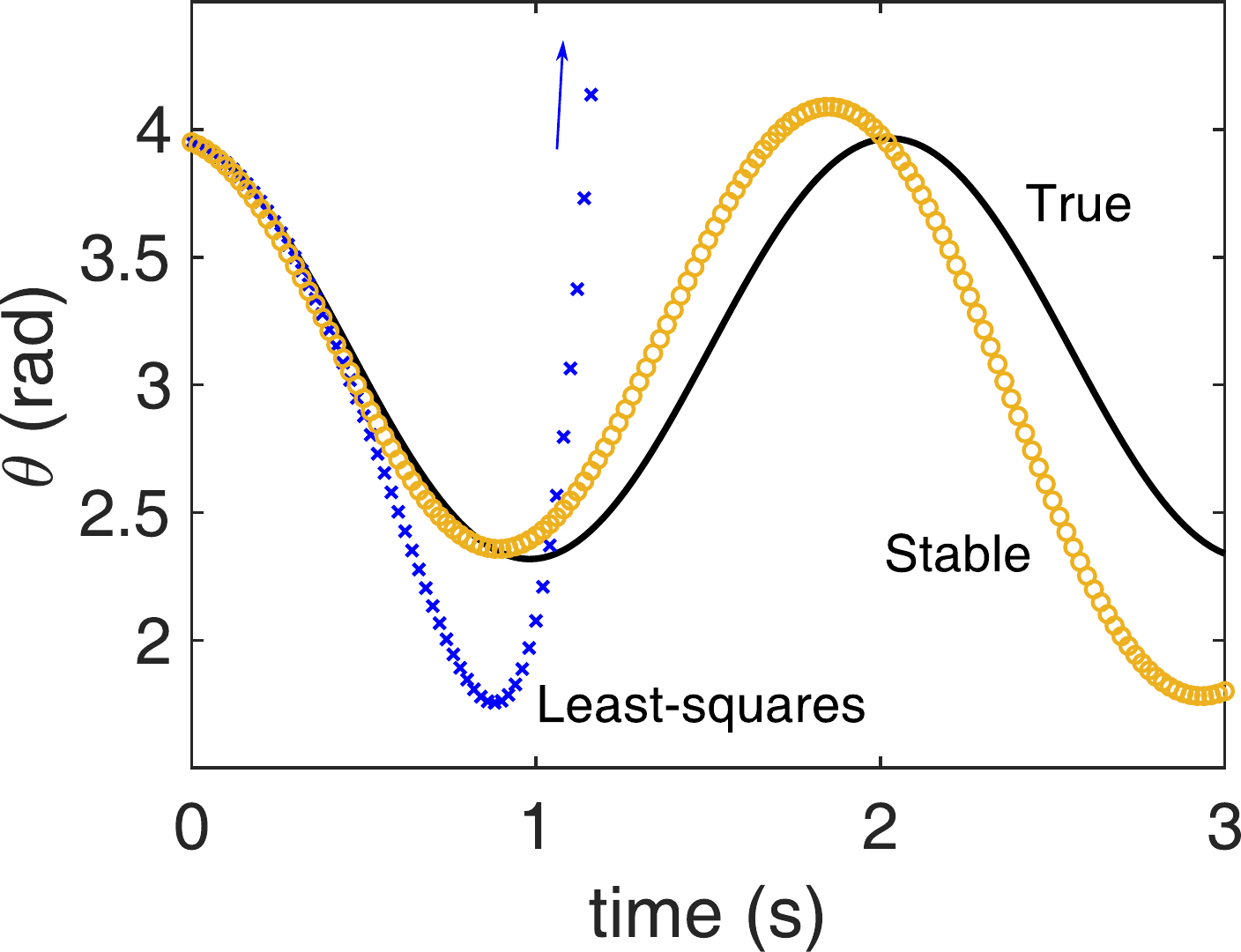}}
 \subcaptionbox{\label{subfig: prediction_unstable_eq}}{
\includegraphics[width = 0.48\columnwidth, keepaspectratio = true]{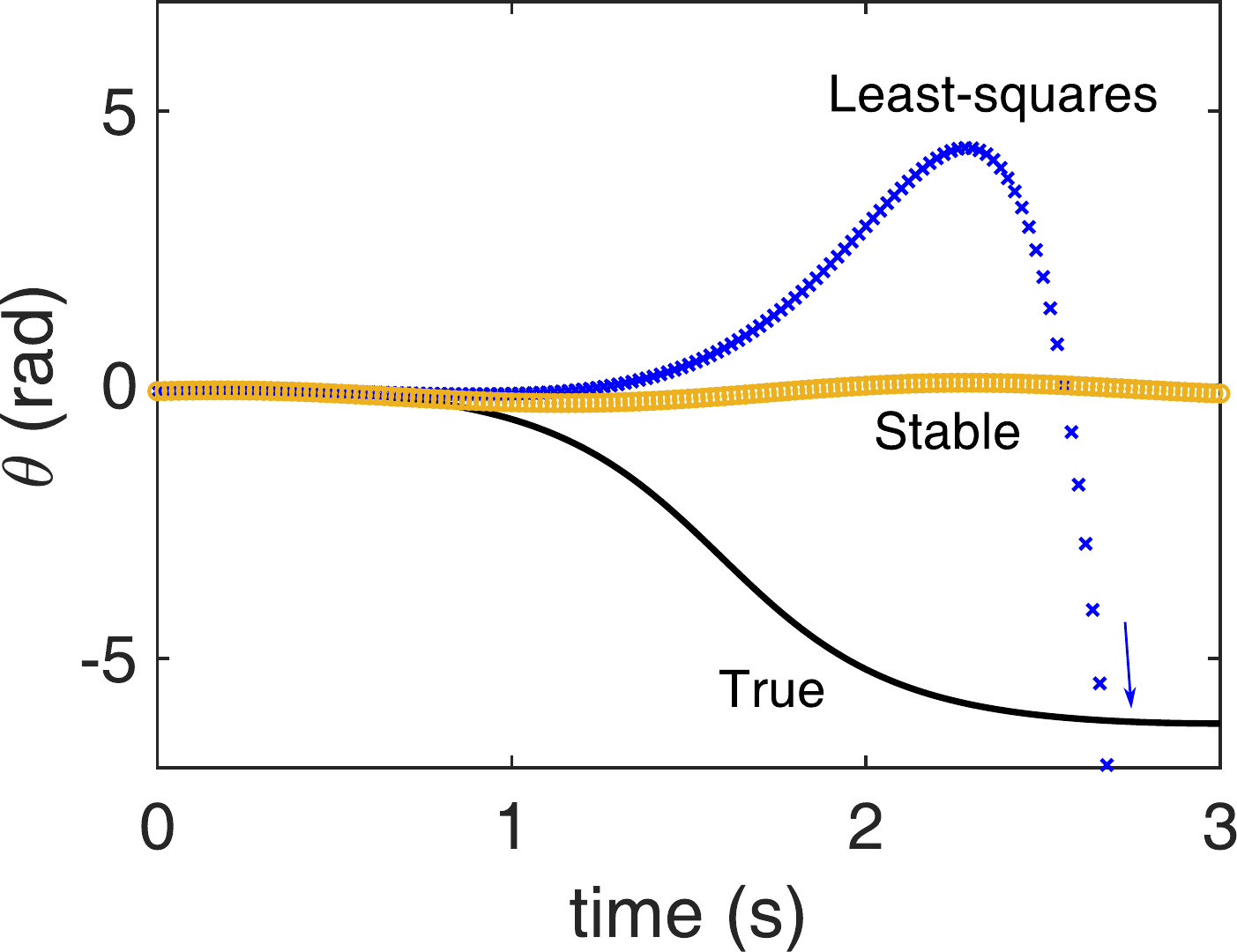}}
 \caption{Figure \ref{subfig: eigenvalues} shows the eigenvalues of the unconstrained, least-squares \eqref{eq:: Kd_LS} and constrained \eqref{eq:: LS_StableKoopman} Koopman operator for the nonlinear dynamics of a pendulum. The constrained operator bounds the unstable eigenvalues within the stability boundary. The stable eigenvalues are also appropriately modified so that constrained Koopman solution locally minimizes the prediction error \eqref{eq:: LS_StableKoopman}. Figures \ref{subfig: prediction} and \ref{subfig: prediction_unstable_eq} show the prediction of the angle of the pendulum system close to the stable and unstable equilibrium points, respectively, using the unconstrained Koopman solution \eqref{eq:: Kd_LS} and the constrained-stable Koopman operator. The predictions of the unconstrained Koopman operator start to diverge away from the pendulum states after 1 second.}\label{fig:: PendulumPrediction}
\end{figure}

Next, we investigate the effect of unstable eigenvalues on the prediction accuracy, as well as the data-efficiency of the algorithms. Specifically, for 300 randomly sampled initial conditions, we compare the prediction error for the pendulum angle as a function of the number of training measurements used to compute a Koopman model. We sample the initial conditions uniformly from  $U(-2\pi,2\pi)$ and $U(-2.5,2.5)$ for the angle and the angular velocity, respectively and training measurements are $\Delta t = 0.02$~s apart; the average error is measured over 3 seconds. We show the results in Fig. \ref{fig:: PendulumDataEfficiency} for an undamped and a damped pendulum system. 

\begin{figure}
    \centering
    \subcaptionbox{Undamped pendulum \label{fig:: Pendulum - AverageError - Unstable vs Stable}}{
    \includegraphics[width = 0.45\columnwidth, keepaspectratio = true]{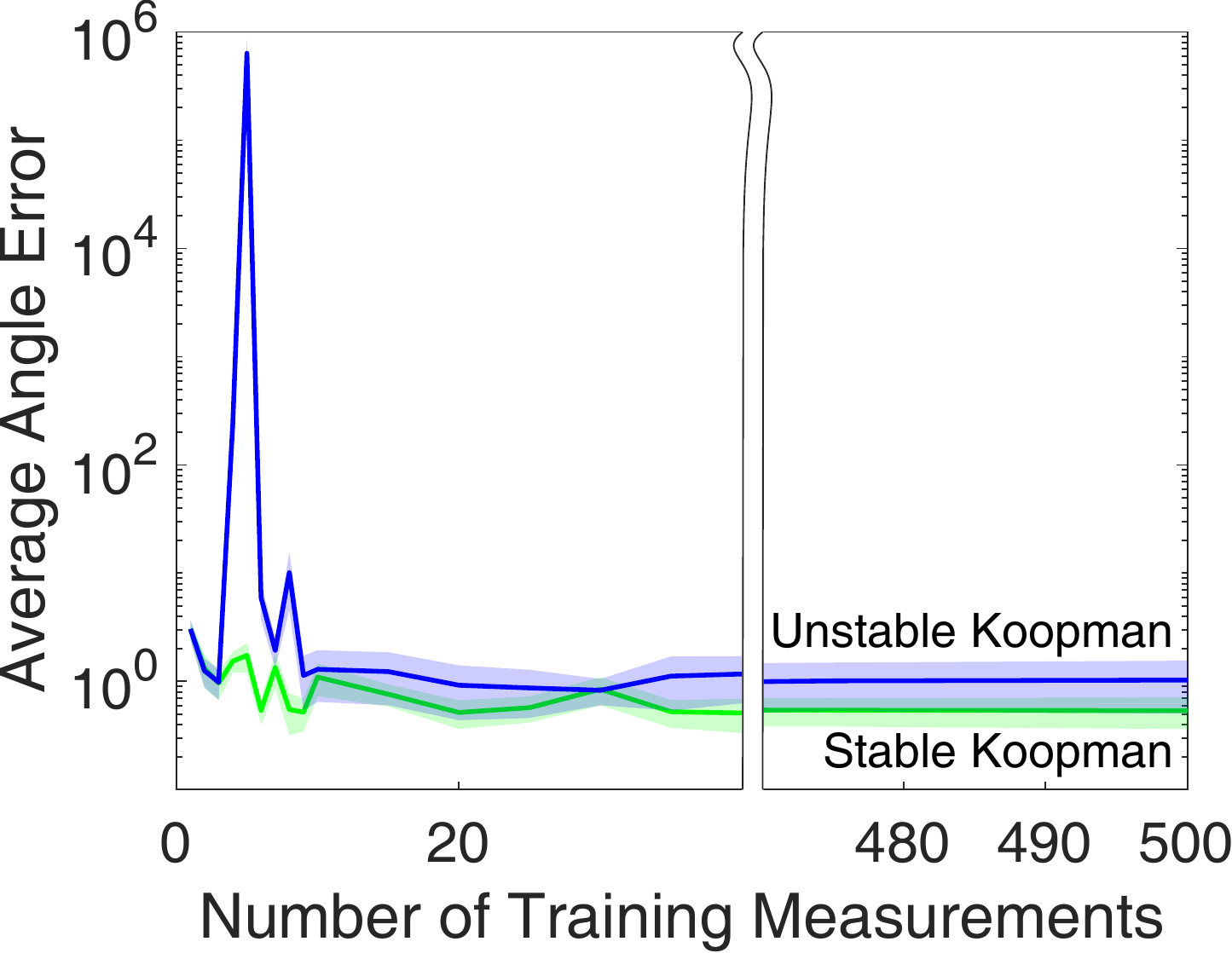}}
    \subcaptionbox{Damped pendulum \label{fig:: DampedPendulum - AverageError - Unstable vs Stable}}{
    \includegraphics[width = 0.51\columnwidth, keepaspectratio = true]{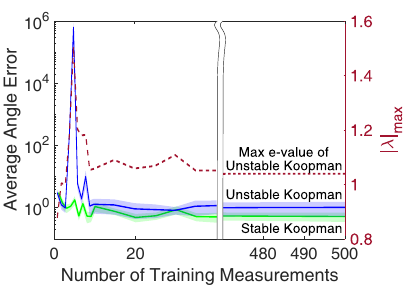}}
    \caption{Average angle error, with one-half standard deviation shading, for the undamped (Fig. \ref{fig:: Pendulum - AverageError - Unstable vs Stable}) and damped (Fig. \ref{fig:: DampedPendulum - AverageError - Unstable vs Stable}) pendulum dynamics, as predicted by the unconstrained and the constrained-stable Koopman operator solutions. For each number of measurements used to compute a Koopman operator, the average angle error is the average absolute difference of the true system state (evolved using the nonlinear dynamics) and the system state as predicted by either Koopman operator over 3 seconds over 300 initial conditions.}\label{fig:: PendulumDataEfficiency}
\end{figure} 

For both systems, the stable Koopman operator leads to smaller average prediction error for any number of measurements used for training, as well as lower error variance than the least-squares, unconstrained solution \eqref{eq:: Kd_LS}, exhibiting better predictive accuracy and robustness. When using the largest training set (500 measurements), the average error is 0.52 radians (about 30 degrees) for the stable Koopman model and 1.03 radians (about 59 degrees) for the least-squares model. Further, the large error spikes associated with the unstable Koopman model around 5 measurements indicate that the unstable solution can be very inaccurate when trained with few measurements; on the other hand, the predictive accuracy of the stable Koopman operator remains almost the same regardless of the size of the training sample, demonstrating data-efficiency and emphasizing the benefit of DISKO in the low-sample limit. When few measurements are available, the least-squares solution is prone to misidentifying the system, but the stability constraints help make the learning process less sensitive to the amount of training data. The envelope of the standard deviation error of the unconstrained Koopman is slightly lower than that of the stable solution around 10 training measurements, suggesting that it is possible that the unstable Koopman generates at times a smaller error and is more accurate. We argue that this is a result of the unstable Koopman overfitting to certain initial conditions and accurately predicting the evolution of very few initial states. Last, the unconstrained least-squares solution for the damped pendulum is always unstable, misidentifying in all cases the true properties of the system. This observation illustrates that, although training with more samples can improve the model accuracy, there is no guarantee that the resulting model will be stable. In fact, data-driven models can often misrepresent stable dynamics with unstable solutions \cite{SparseData_Koopman, Robust_Koopman_v2}.

Next, we use the undamped pendulum system to illustrate how stable Koopman operators can be used to construct Lyapunov functions and verify the stability of a controller. Using LQR feedback, we generate a trajectory and use the state measurements and the DISKO algorithm to compute a stable Koopman operator. We then use the Koopman operator to construct a Lyapunov function, which we evaluate with the measurements from the controlled trajectory. 

To generate LQR control, we use $Q = \text{diag}[1, 1]$ and $R = \text{diag}[0.01]$, initial conditions $[\theta_0, \dot{\theta}_0] = [\pi, 5]$ and collect measurements every $\Delta t = 0.1$~s. We train a Koopman operator using $\Psi(s) = [\theta, \dot{\theta}, \theta^2, \dot{\theta}^2, \sin(\theta), \sin(\dot{\theta}), \sin(\theta) \dot{\theta}, \sin(\dot{\theta})\theta]^T$, which satisfy the stability conditions presented in Section \ref{sec:: Stability_properties_of_Koopman}. We then solve the Lyapunov equation \eqref{eq:: LyapunovEquation} using $Q_{\tilde{\mathcal{K}} } = \mathbb{I}_{8\times 8}$ (identity matrix) and construct the candidate control-Lyapunov function as $V(\Psi((s)) = \Psi(s)^T P \Psi(s)$. We show the results in Fig. \ref{fig:: LyapunovPendulum}. The candidate control-Lyapunov function evaluated with the controlled trajectory of the pendulum satisfies the properties of a Lyapunov function and shows that the specific trajectory is converging to the equilibrium. The validity of the Lyapunov function rests on the assumptions shown in \eqref{eq:: a_max} and \eqref{eq:: boundedError}, respectively, which place upper and lower bounds on $a$ respectively;  \eqref{eq:: zeroError} is satisfied because all chosen basis functions evaluate to zero at the equilibrium. These bounds, used to test the validity of the candidate Lyapunov function, help answer that the controlled dynamics provably converge to the equilibrium whenever there exist solutions that satisfy both assumptions.

\begin{figure}
\centering
\subcaptionbox{Candidate control-Lyapunov function}{\includegraphics[width = 0.8\linewidth, keepaspectratio = true]{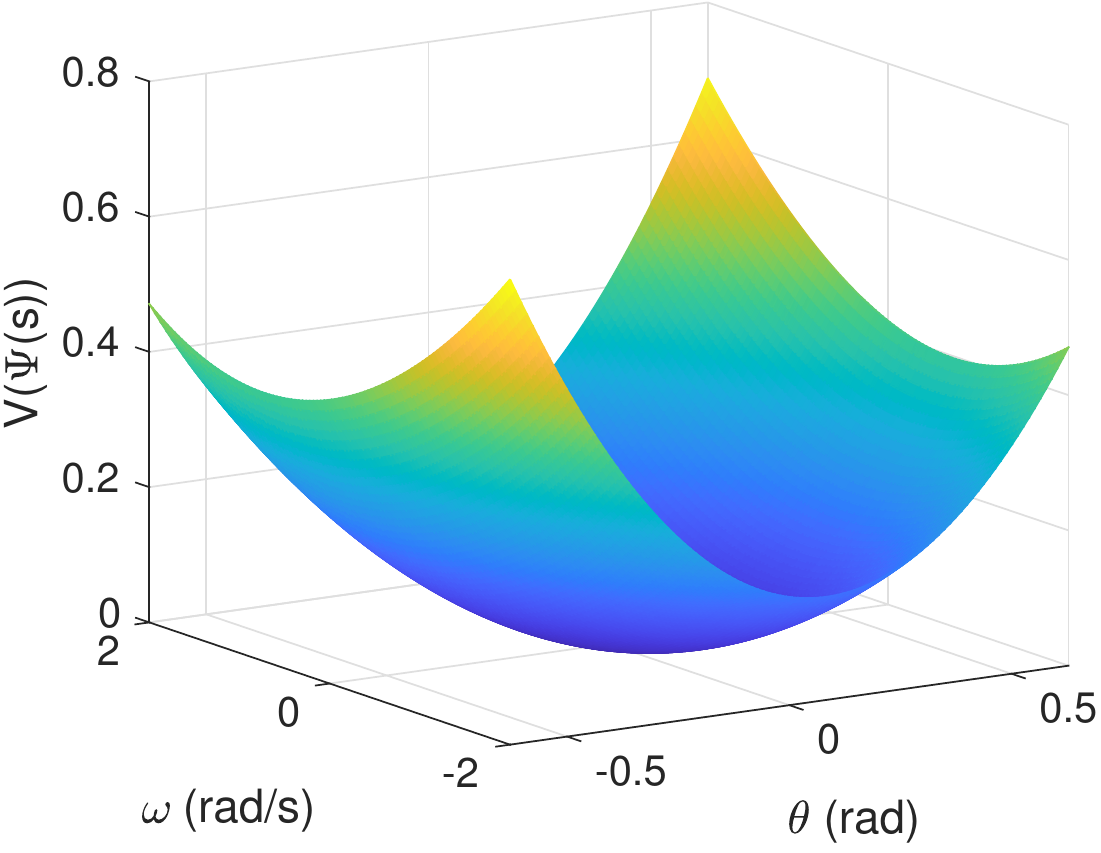}}
\subcaptionbox{Test of assumptions for validity of Lyapunov function}{\includegraphics[width = 0.99\linewidth, keepaspectratio = true]{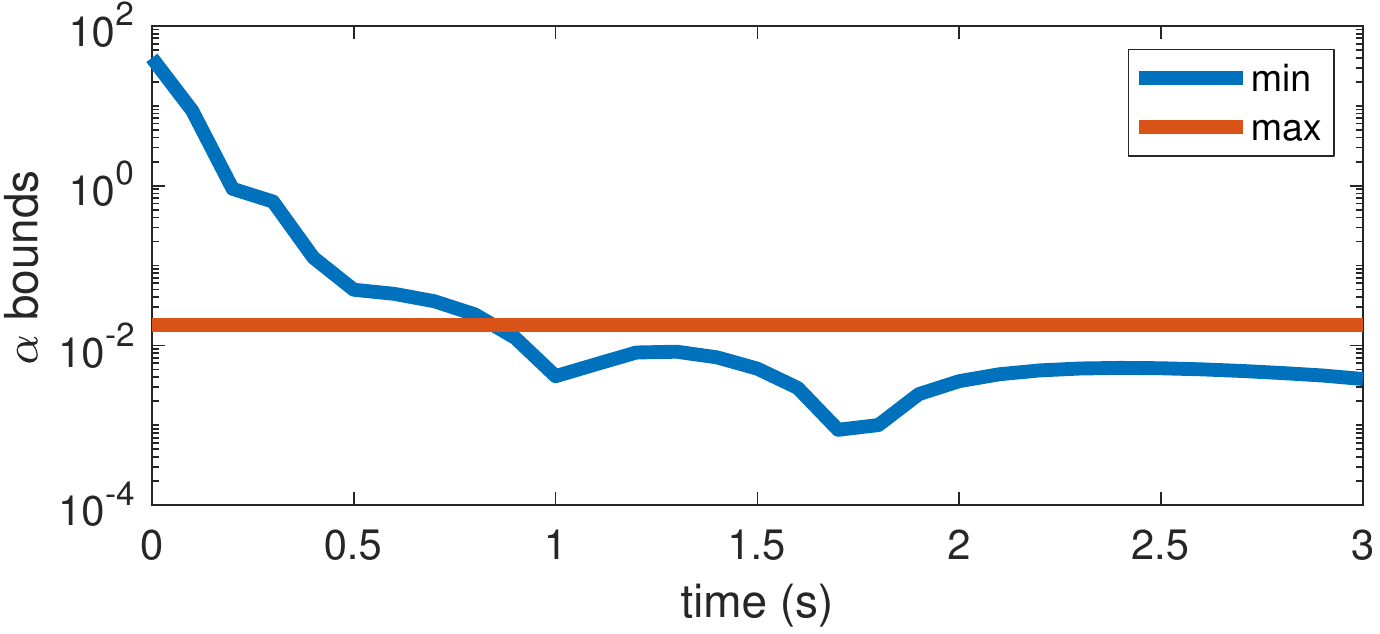}}
\subcaptionbox{Evaluation with state measurements from LQR-controlled trajectory.}{\includegraphics[width = \linewidth, keepaspectratio = true]{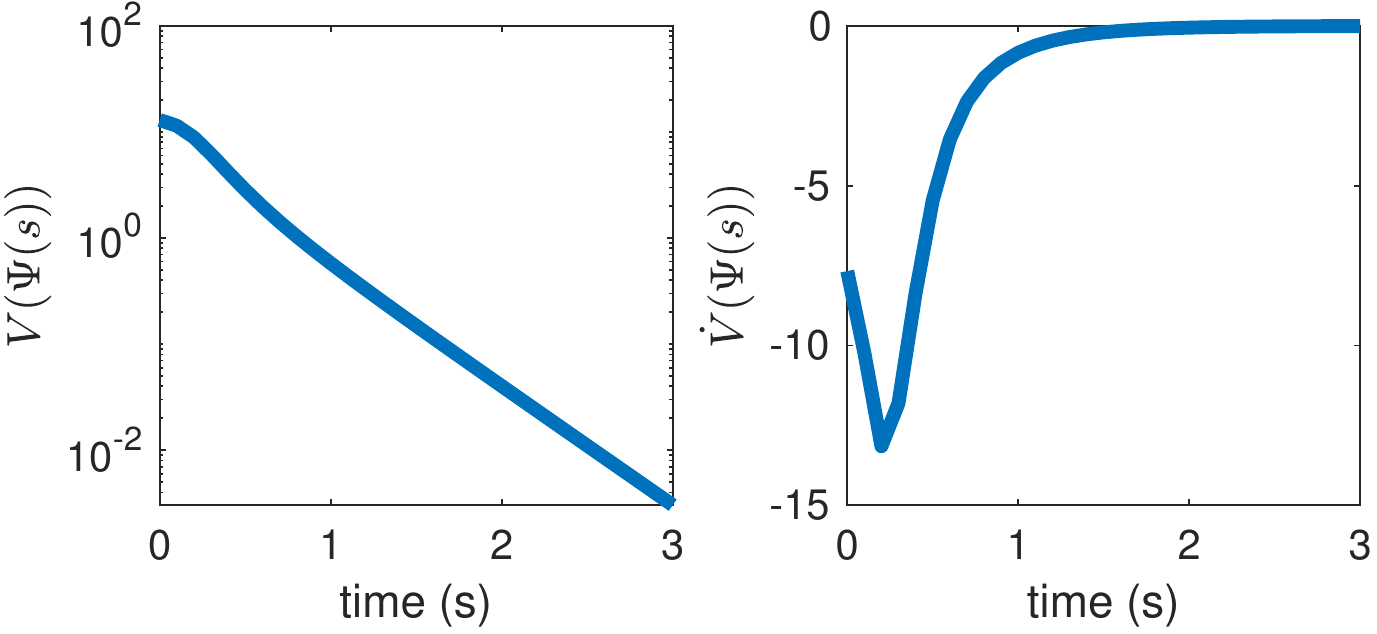}}

 \caption{Candidate control-Lyapunov function constructed from stable Koopman operators and evaluated on a LQR-controlled pendulum. The candidate control-Lyapunov function is used to verify the stability of the controlled trajectory. The existence of a solution that satisfies the assumptions \eqref{eq:: a_max} and \eqref{eq:: boundedError} required for the validity of the Lyapunov function is tested; \eqref{eq:: zeroError} is satisfied because all chosen basis functions evaluate to zero at the equilibrium. In the last two seconds, there exist solutions for $a$ that satisfy the upper and lower bounds from \eqref{eq:: a_max} and \eqref{eq:: boundedError}, respectively, and the constructed solution is provably a valid Lyapunov function. }\label{fig:: LyapunovPendulum}
\end{figure}

\begin{figure}
\centering
\subcaptionbox{Test of assumptions for validity of Lyapunov function}{\includegraphics[width = 0.99\linewidth, keepaspectratio = true]{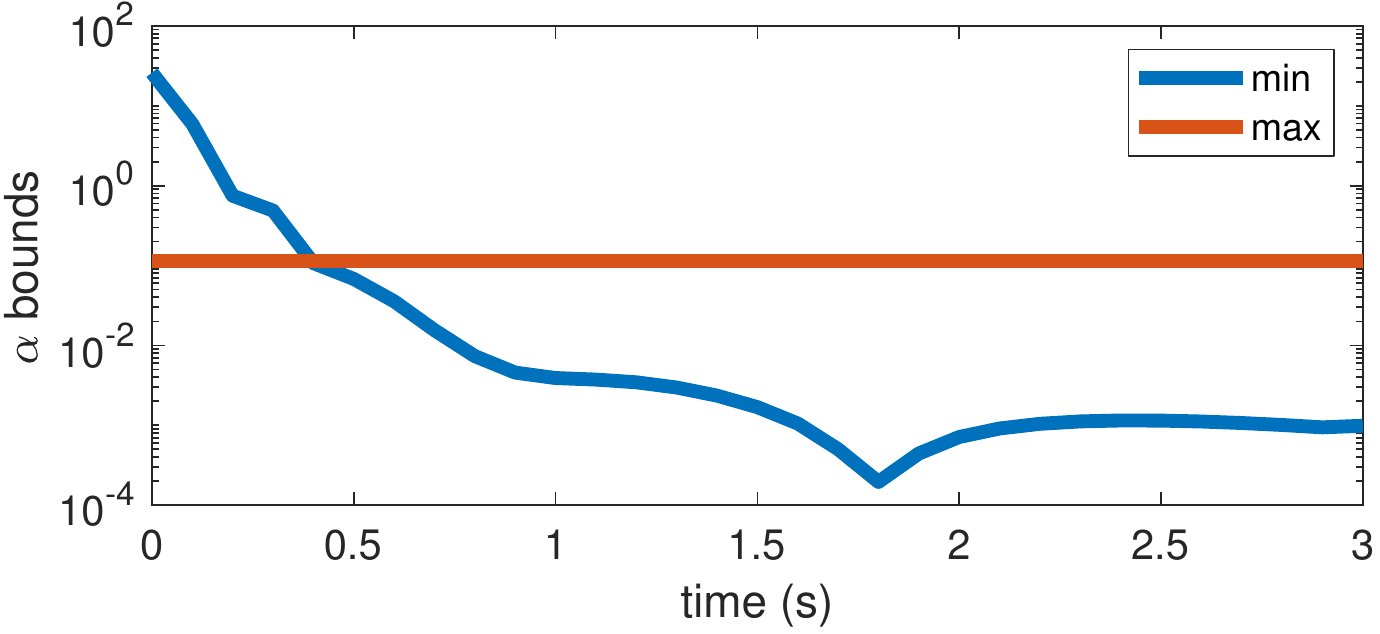}}
\subcaptionbox{Evaluation with state measurements from LQR-controlled trajectory.}{\includegraphics[width = \linewidth, keepaspectratio = true]{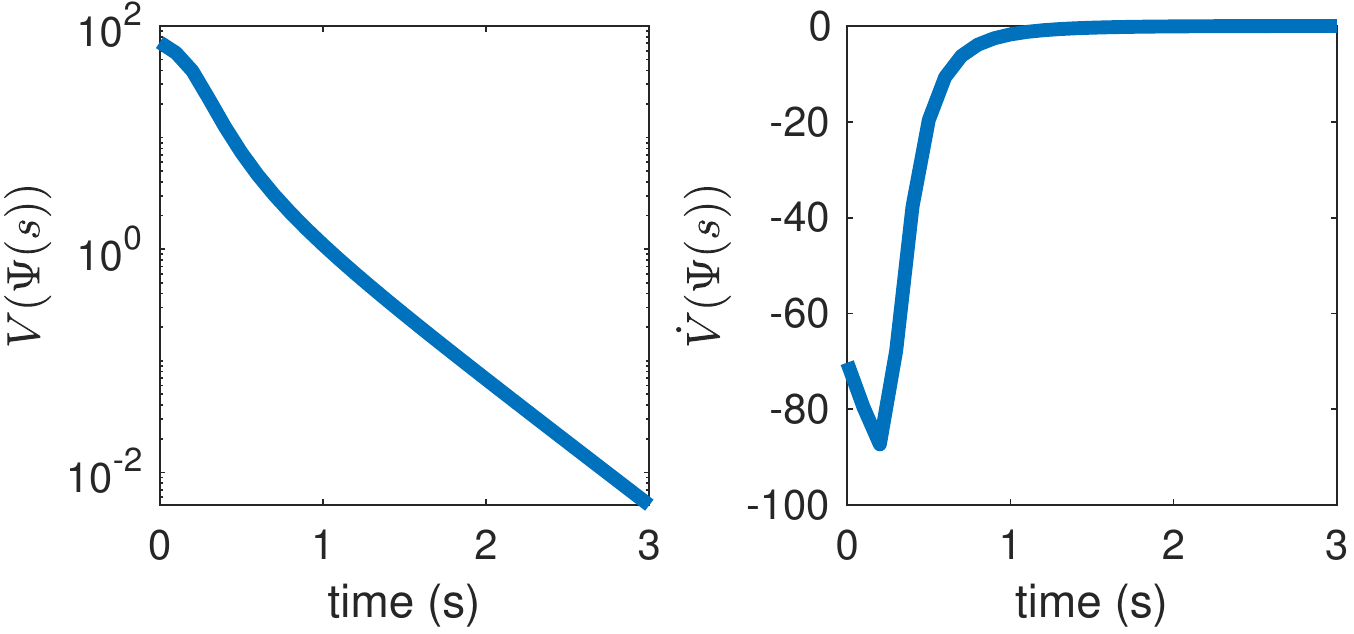}} \caption{Evaluation of the candidate Lyapunov function with the same data as in Fig. \ref{fig:: LyapunovPendulum}, but improved basis function selection. Improving the accuracy of the basis functions leads to relaxed upper and lower bound constraints and a longer period over which the candidate Lyapunov function is validated. }\label{fig:: LyapunovPendulum_Haseli}
\end{figure}

\begin{figure}
\centering
\subcaptionbox{Test of assumptions for validity of Lyapunov function}{\includegraphics[width = 0.99\linewidth, keepaspectratio = true]{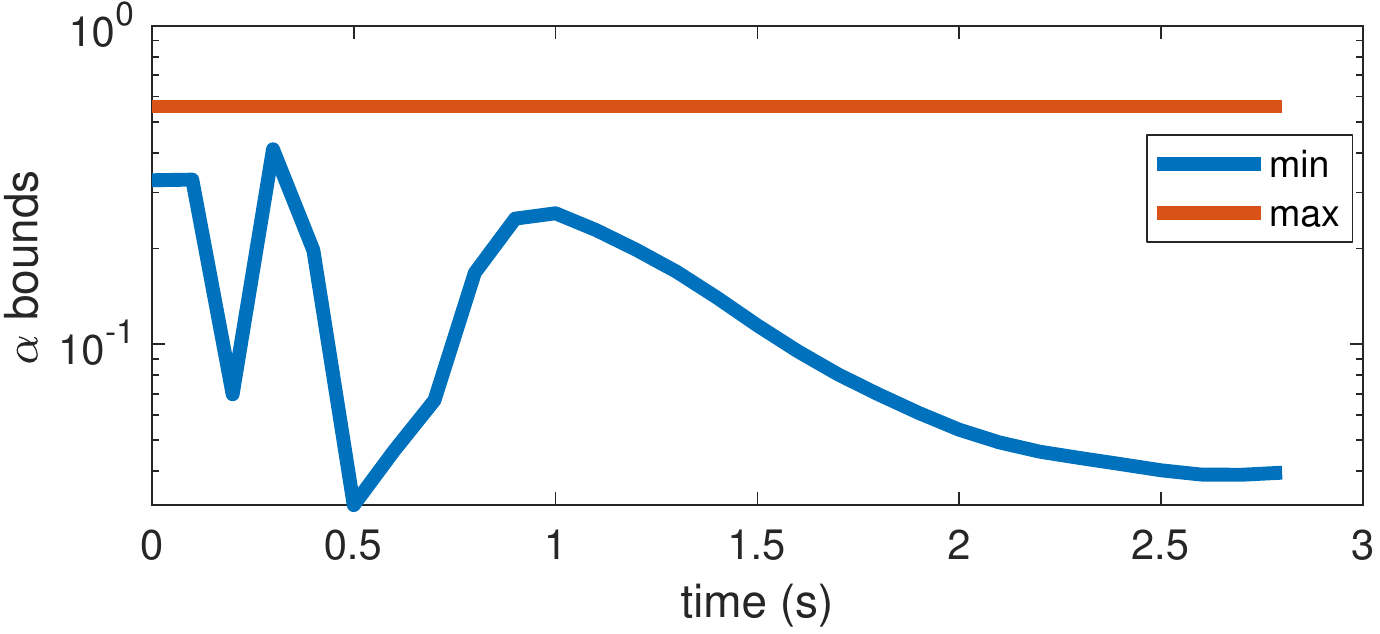}}
\subcaptionbox{}{\includegraphics[width = \linewidth, keepaspectratio = true]{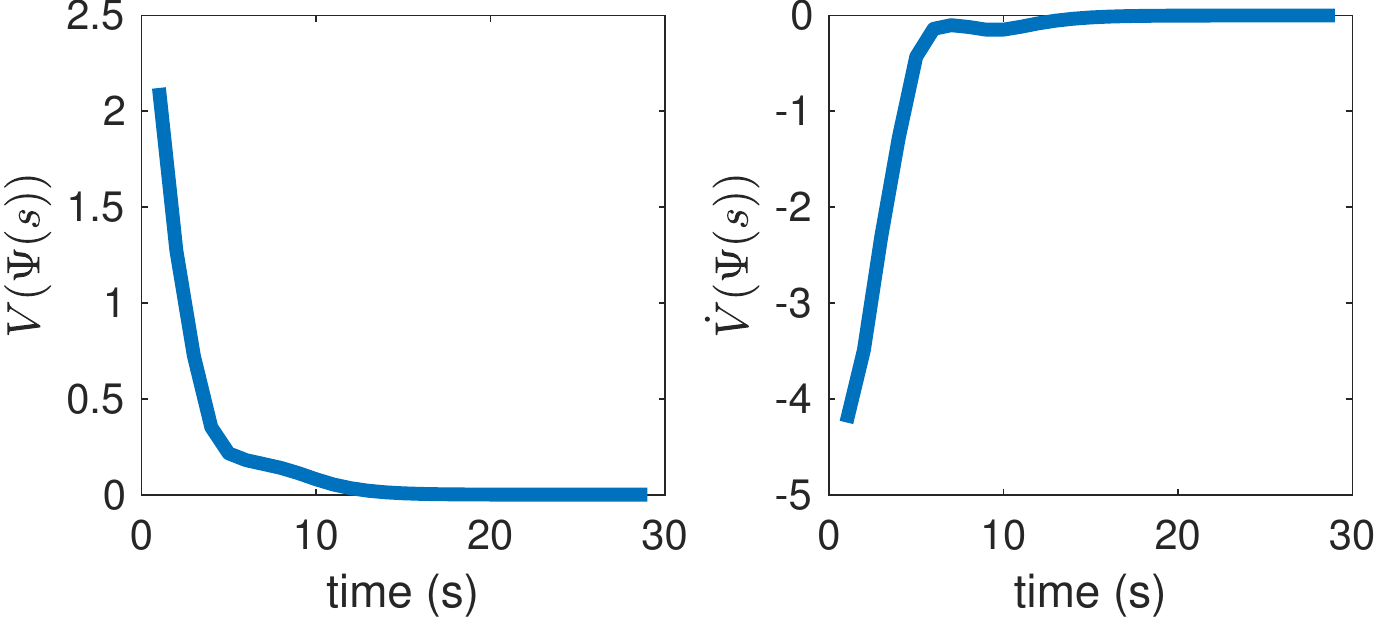}} \caption{Evaluation of the candidate Lyapunov function for the cart-pendulum dynamics.}\label{fig:: LyapunovPendulum_cart}
\end{figure}

Note that the control-Lyapunov function shown in Fig. \ref{fig:: LyapunovPendulum} is used to verify the stability of the applied controller for the particular trajectory and may not be a Lyapunov function everywhere in the state space. According to Theorem 3, a stable Koopman operator can provably generate a Lyapunov function provided the model error satisfies the conditions in \eqref{eq:: a_max}, \eqref{eq:: zeroError}, and \eqref{eq:: boundedError}. For controlled systems, one can construct a candidate control-Lyapunov function from the identified stable Koopman. The candidate control-Lyapunov function is guaranteed to be valid everywhere (that is, in those parts of the state space) that the model error using the stable Koopman satisfies \eqref{eq:: a_max}, \eqref{eq:: zeroError} and \eqref{eq:: boundedError}, the sufficient conditions. 
In practice, however, the candidate control-Lyapunov function evaluated along the controlled trajectory may still be valid, which can serve as a certificate for stability even if \eqref{eq:: a_max}, \eqref{eq:: zeroError}, and \eqref{eq:: boundedError} are not satisfied. Further, one could solve for a control policy that generates a stable Koopman operator and generates a valid control-Lyapunov function, but this is left for future research.

Last, the choice of basis functions explicitly affects the search of valid Lyapunov functions through \eqref{eq:: zeroError} and \eqref{eq:: boundedError}. Basis functions that evaluate to zero at the equilibrium to satisfy \eqref{eq:: zeroError} and also improve the ratio between the model error and their norm to relax the constraint in \eqref{eq:: boundedError} can help validate candidate Lyapunov functions. This is why the properties of the basis functions in Proposition \ref{prop: conditions_for_Koopman_basis} are needed: to improve the model error and satisfy the required conditions so that a candidate Lyapunov function can be found. To illustrate the effect of the basis functions, we use the algorithm in \cite{Haseli_EDMDtunableAccuracy, haseli2021generalizing} to improve the accuracy of the basis functions selection. We show the results in Fig. \ref{fig:: LyapunovPendulum_Haseli}. The new choice of Koopman basis functions both decrease the lower bound constraint \eqref{eq:: boundedError} and increase the upper bound constraint \eqref{eq:: a_max}, thus confirming the validity of the Lyapunov function and the fact that the dynamics are stabilizing to the equilibrium for a longer period: from around 0.5 seconds till the end.

Last, we repeat these results for the inverted cart-pendulum dynamics \cite{mamakoukas2018superlinear}. To collect training data, we invert the pendulum using the Sequential Action Controller (SAC) and the system dynamics and parameters shown in \cite{mamakoukas2018superlinear} (Section IV. B). We train a Koopman operator using 
\begin{align*}
    \Psi(s) = [\theta, \dot{\theta}, x, \dot{x}, \theta^2, \dot{\theta}^2, x^2, \dot{x}^2, \sin(\theta), \sin(\dot{\theta}), \sin(x), \sin(\dot{x}), \\ \sin(\theta) \dot\theta, \sin(\dot\theta) \theta, \sin(x) \dot x, \sin(\dot x) x, \theta \dot\theta, \theta x, \theta \dot x, \dot\theta x, \dot\theta \dot x, x \dot x ]^T,
\end{align*}
which satisfy the stability conditions presented in Section \ref{sec:: Stability_properties_of_Koopman} for the equilibrium at $s_e = [0, 0, 0, 0]^T$ and solve the Lyapunov equation \eqref{eq:: LyapunovEquation} using $Q_{\tilde{\mathcal{K}} } = \mathbb{I}_{22 \times 22}$. We show the results in Fig.~\ref{fig:: LyapunovPendulum_cart}. The candidate Lyapunov function is validated for all the measurements of the controlled trajectory data. 

\subsection{Nonlinear Control Using Stable Koopman Operators}
In this section, we demonstrate the benefit of using stable Koopman operators for nonlinear control. By improving the robustness and modeling accuracy of data-driven models, we argue that stability-constrained models would also lead to improved control performance. 

Knowledge of the equilibrium points for a system can also be used to filter basis functions \textit{a-priori}, without any training data. We provide an illustrative example in Appendix \ref{app:: FilteringOfBasisFunctions}. Leveraging the conditions for the Koopman models as shown in Proposition \ref{prop: conditions_for_Koopman_basis} to optimize the choice of basis functions is a promising research topic that merits further investigation and is deferred to future work. 

\subsubsection{Quadrotor}
We first consider stabilizing a falling quadrotor. Using active learning, which has been shown to enhance learning and the accuracy of identified dynamics \cite{Ian_active_learning}, we collect training data within the first second of the free-fall. Then, using the same training sample, we compute a Koopman model using the least-squares solution \eqref{eq:: Kd_LS} and DISKO and develop an LQR policy to stabilize the quadrotor. The system feedback rate is $200$~Hz and the state is partially observed containing only the measured body-relative gravity vector and the body linear and angular velocities (see~\cite{Ian_active_learning} for more detail). The quadrotor dynamics, LQR parameters, and Koopman basis functions are the same as in~\cite{Ian_active_learning}. Note that the quadrotor dynamics used in \cite{Ian_active_learning}, without control, have no equilibrium points due to gravity. Without gravity (zero gravity term) all the functions used here evaluate to zero, imposing a Koopman equilibrium at the same equilibrium as the original dynamics (when the linear and angular velocities are zero). Even without an equilibrium point of the free dynamics, imposing that the learned model is stable numerically improves the prediction and leads to better control performance as shown in Figure~\ref{subfig:: quadPrediction}.

\begin{figure}
\centering
\begin{subfigure}[]{0.49\columnwidth}
\subcaptionbox{Control performance with increased planning horizon. \label{subfig:: 30timeStep}}{\includegraphics[width =\linewidth, keepaspectratio = true]{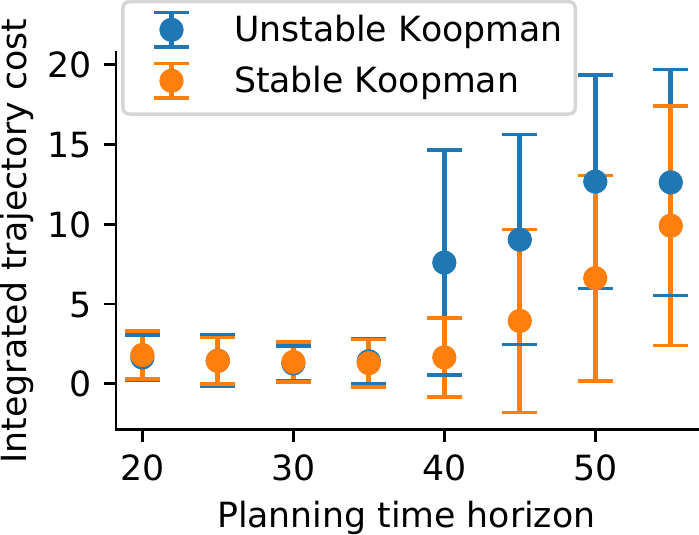}}

\par\medskip
\subcaptionbox{40 time-steps horizon. \label{subfig:: 40timeStep}}{\centering\includegraphics[width =  \linewidth, keepaspectratio = true]{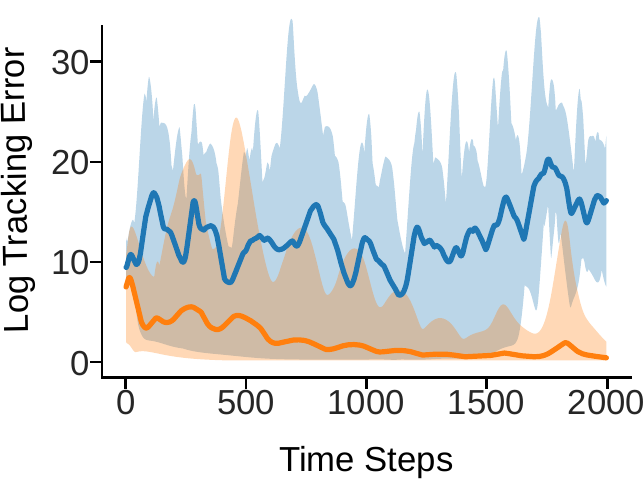}}
\end{subfigure}
\begin{subfigure}[]{0.49\columnwidth}
\subcaptionbox{Prediction error. \label{subfig:: quadPrediction}}{\includegraphics[trim=0 0 0 0,clip,width =0.9\linewidth, keepaspectratio = true]{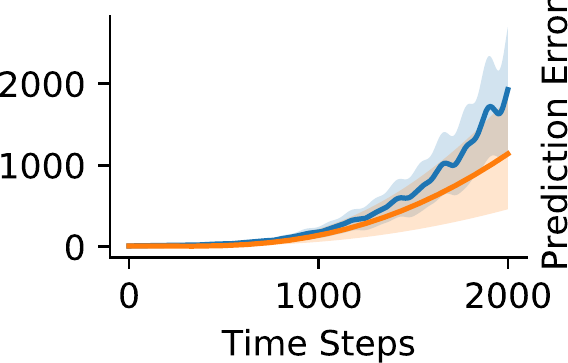}}
\subcaptionbox{LQR-control trajectories. \label{subfig:: quadTrajectory}}{\includegraphics[trim=0 0 0 0,clip,width =0.85\linewidth, keepaspectratio = true]{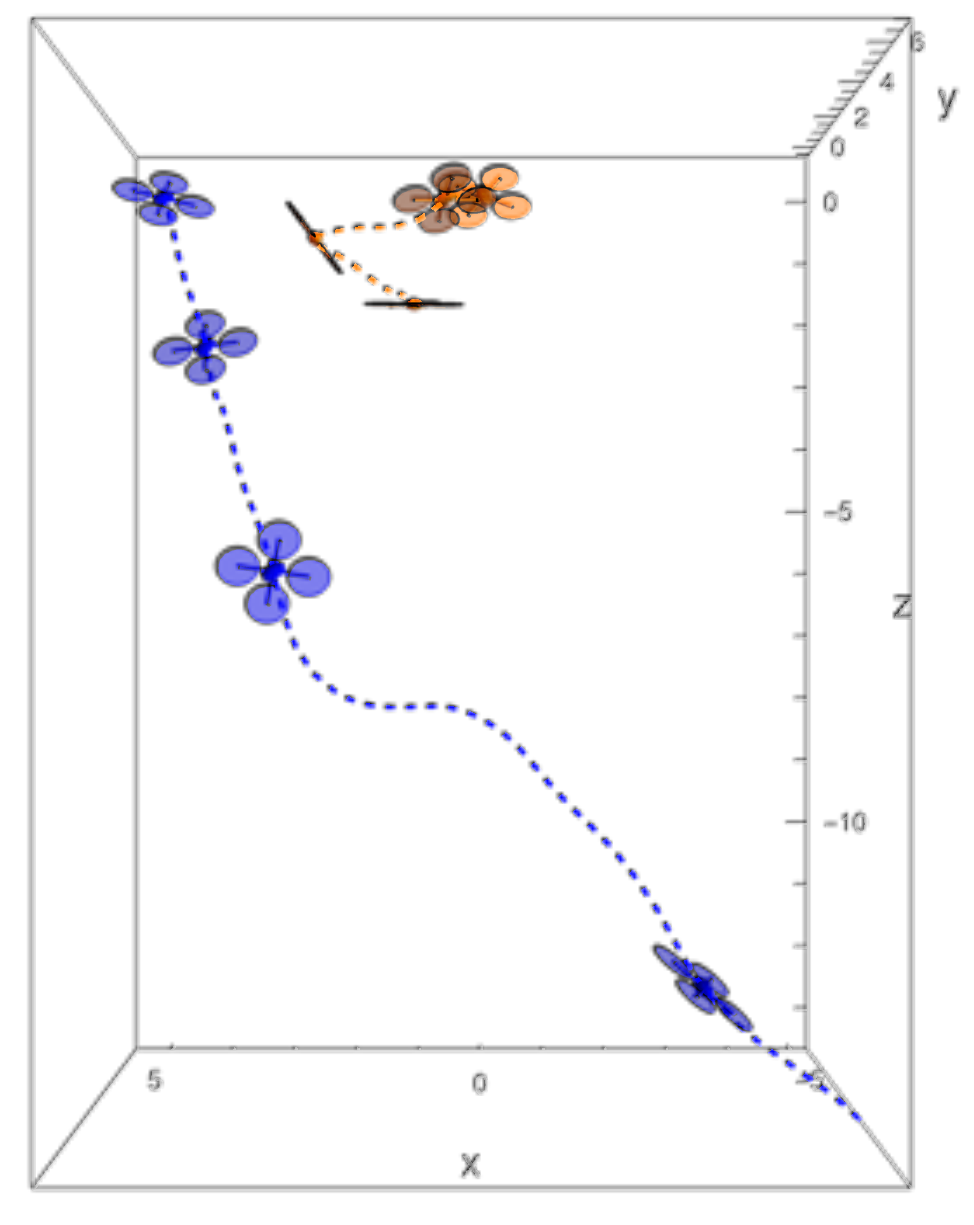}}
\end{subfigure}
\caption{Performance of LQR control derived from the stable \eqref{eq:: LS_StableKoopman} and least squares \eqref{eq:: Kd_LS} Koopman operators for the quadrotor dynamics. Both models use the same training measurements that are collected with active learning. At the end of the learning phase, the stable Koopman is computed and the LQR gains from both models are derived. Figure \ref{subfig:: 30timeStep} illustrates control performance as a function of planning horizon and Figure \ref{subfig:: 40timeStep} show the log error of the tracking cost for 10 trajectories with the same uniformly sampled initial conditions. The solid line represents the median score of each approach and the shaded envelope the lowest and highest cost. Figure \ref{subfig:: quadPrediction} illustrates the prediction error of the learned Koopman operators compared to ground truth. Figure \ref{subfig:: quadTrajectory} shows a trajectory using the 40 time-step horizon control. The initial conditions are the same, but shifted in the x-axis for better visibility.}\label{fig:: ActivelyLearned_Quadrotor}
\end{figure}

We present the stabilizing performance of the stable and least squares Koopman models in Fig. \ref{fig:: ActivelyLearned_Quadrotor}. We consider a sequence of planning time horizons used in computing the finite-horizon LQR control and illustrate the tracking error for 40 time step horizon using the same 10 uniformly sampled initial conditions as done in \cite{Ian_active_learning}. We use the median score (in log scale) as a performance metric of the two approaches, because it is not as biased in a case of failure (when states diverge). Using DISKO, the control is robust and stabilizes the dynamics for all choices of the prediction horizon, contrary to the unconstrained model that fails when using a longer prediction window. Note that the time horizon is typically chosen to be as long as possible while satisfying application-specific computational demands in order to enable more far-sighted control solutions \cite{droge2011adaptive}. Here, as the time horizon is increased, the stable model leads to lower error than the unconstrained model with smaller performance variance. These results show that active learning techniques that enhance learning are not sufficient to address the challenges of unconstrained data-driven models and can be further improved through the use of DISKO. In addition, these results highlight that unstable models can quickly become impractical to use and provide motivation as to why stability constraints can be useful even when modeling unstable dynamics.

\subsubsection{Pusher-slider system}
Next, we demonstrate the benefits of DISKO using a pusher-slider system \cite{pusherSlider_feedbackControl, pusherSlider_probabilistic, pusherSlider_dataEfficient}. The pusher is a steel rod held tightly by the end effector of the Franka Emika Panda robot \cite{frankaEmikaPanda} and the slider is a rectangular block with dimensions $15.3 \times 13$~cm. We record the states of the system at $10$~Hz using an overhead camera and QR codes on the slider. We show the experimental setup in Fig. \ref{fig:: pusherSliderSetup}.

 To train a model, we collected data using a controller to push the block with the end-effector of the robot. We collected six training sets (200 measurements each) performing random maneuvers and used the data to compute a least-squares unconstrained and a stable Koopman model. The basis functions used are
 \begin{align*}
     \Psi(s) =& [x, y, \theta, p_y, v_n, v_p, \sin(\theta) v_n, \cos(\theta) v_n, \sin(\theta) v_p, \\
     & \cos(\theta) v_p, p_y v_n, v_p v_n]^T,
 \end{align*}
 where $x, y,$ and $\theta$ are the world-frame coordinates and orientation of the block, $p_y$ is the distance of the pusher (end effector of the Franka Emika robot) away from the center of the block and along its pushing side (we assume that the pusher is always in contact with the block), $v_n$ and $v_p$ are the normal and parallel velocity of the slider in the body-frame of the block, respectively, and $u \in \mathbb{R}^2$ is the acceleration input for the body-frame normal and parallel velocity of the slider.
 
Note that the equilibrium points of the original dynamics occur when the velocities of the block are zero, that is when $v_n = v_p = 0$. The basis functions chosen here may violate the condition in Proposition \ref{prop: conditions_for_Koopman_basis}, since $x, y, \theta$, and $p_y$ may not be zero at the equilibrium. However, we included the system states in the basis functions to design control that drives the system to the desired configuration, as specified by $x, y, \theta, $ and $p_y$, and not simply bring it to rest. By subtracting the desired values from the system states in the basis functions (i.e. $x' = x - x_{des}$), it is possible to impose that an equilibrium of the original dynamics at the desired configuration is also an equilibrium of the Koopman dynamics (see Proposition \ref{prop: conditions_for_Koopman_basis}). However, subtracting the target configuration from the Koopman basis functions would require that a Koopman model is re-trained for every different target, which we avoided for simplicity. Using information about the target in the basis functions to ensure that Proposition \ref{prop: conditions_for_Koopman_basis} holds is left for future work.
 
 To compare the unconstrained and stable least-squares models, we forward predict the system with zero inputs (see Fig. \ref{fig:: PusherSlider_ZeroControlPrediction}). Given no movement from the pusher, the block should stay in place. However, the states propagated with the unconstrained model diverge, as expected for an unstable linear model. On the other hand, the simulated prediction of the DISKO model barely shows any motion and is consistent with the expected behavior of the system. This difference in performance highlights again the numerical motivation for stability. Further, imposing the right properties on learning helps narrows the solution space and identify representations that mirror the true system properties and generalize beyond training.
 \begin{figure}
    \centering
    \includegraphics[width = 0.9\columnwidth, keepaspectratio = true]{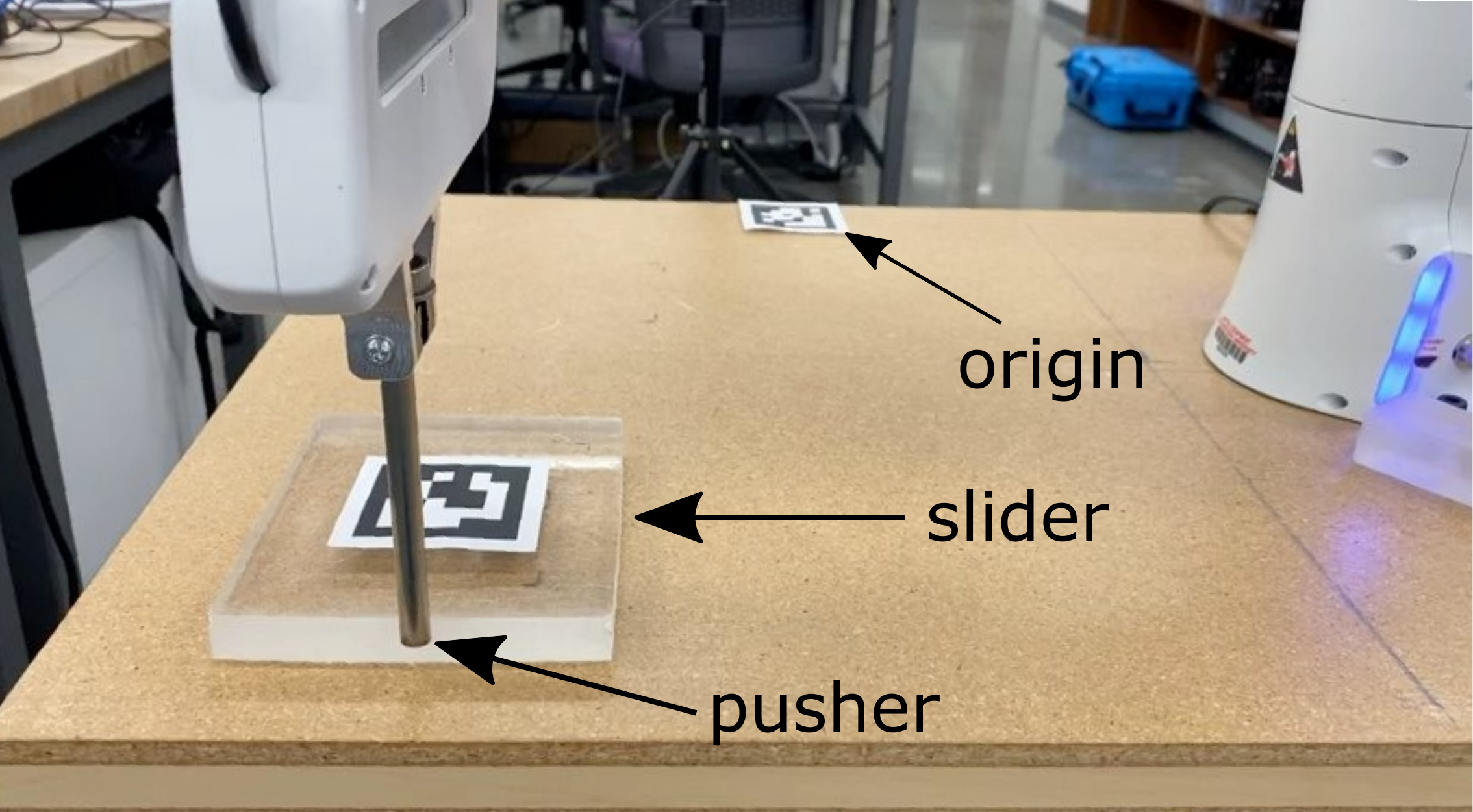}
    \caption{Experimental setup of the pusher-slider system. We record the states of the pusher and the slider with an overhead camera and identify the block configuration using QR labels.}\label{fig:: pusherSliderSetup}
\end{figure}
 
 \begin{figure}
    \centering
    \subcaptionbox{LS Koopman}{\includegraphics[width = 0.49\linewidth, keepaspectratio = true]{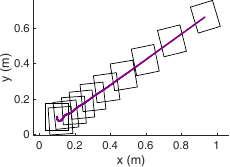}}
    \hfill
    \subcaptionbox{DISKO}{\includegraphics[width = 0.49\linewidth, keepaspectratio = true]{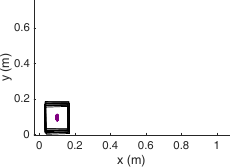}}
    \caption{Simulation of the pusher-slider system over 500 time steps ($ dt = 0.1$) with zero control inputs. The least-squares Koopman model is unstable and drifts away, despite the fact that there should be no motion in the absence of control.}\label{fig:: PusherSlider_ZeroControlPrediction}
\end{figure}
\begin{figure*}
    \centering
    \includegraphics[width = \textwidth, keepaspectratio = true]{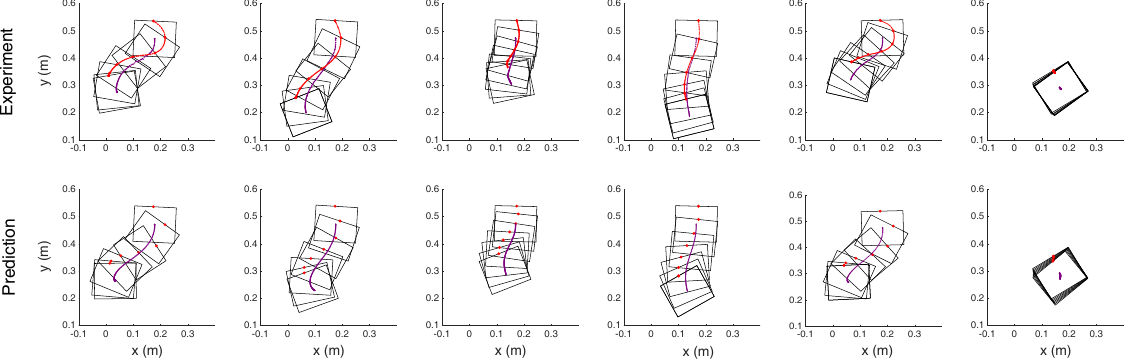}
    \caption{Comparison of the experimental trajectories from the training data to those obtained in simulation using DISKO. Instances are shown every 2 seconds. We initiate the simulated trajectories with the starting states of the experiments and forward-propagate them with the same control inputs that were applied in each run. We indicate the position of the pusher with red and the center of the slider with purple. In the predictions, the location of the pusher is known only as $py$, the distance away from the center of the block along its pushing side; we assume that the pusher is always in contact with the block. To highlight this difference, we do not plot a trajectory of the pusher, but show its location only at the instants the block is drawn. }	\label{fig:: pusherSlider_ExpVsPrediction} 
\end{figure*}

Next, we test the predictive accuracy of the learned DISKO representation against the training. Specifically, we use the stable Koopman model to forward-simulate the pusher-slider system with the controls applied during the experiments and compare it to the actual trajectories. We show the results in Fig. \ref{fig:: pusherSlider_ExpVsPrediction}. In all cases, the model obtained using DISKO generates qualitatively similar trajectories compared to the actual experiments. 

\begin{figure*}
 \subcaptionbox{}{
 \includegraphics[height = 0.4\columnwidth, keepaspectratio = true]{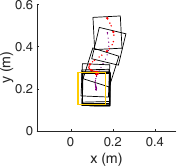}}\vspace{1ex}
 \hfill
 \subcaptionbox{}{
\includegraphics[height = 0.4\columnwidth, keepaspectratio = true]{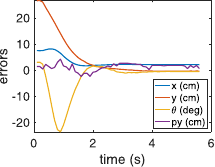}}
 \hfill
 \subcaptionbox{}{
\includegraphics[height = 0.4\columnwidth, keepaspectratio = true]{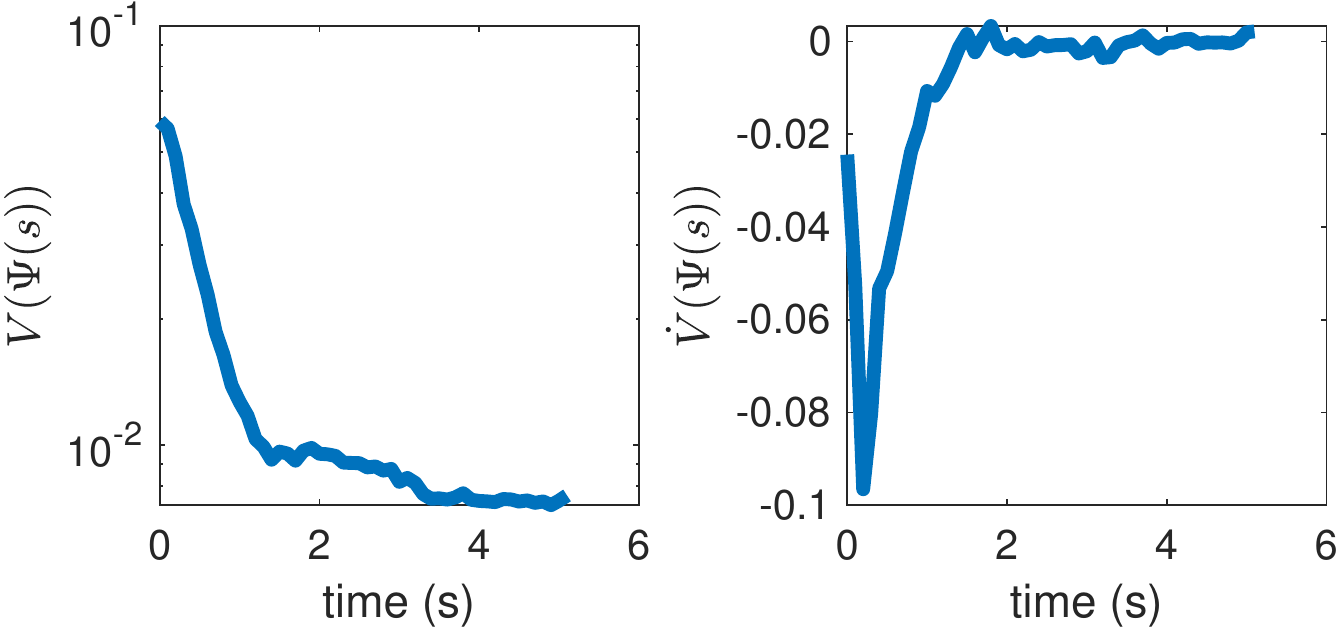}}
\vspace{1ex}
 \subcaptionbox{ }{
 \includegraphics[height = 0.4\columnwidth, keepaspectratio = true]{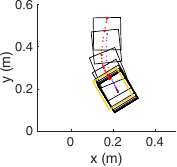}}
 \hfill
 \subcaptionbox{}{
\includegraphics[height = 0.4\columnwidth, keepaspectratio = true]{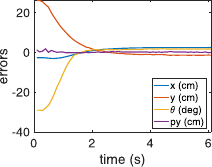}}
 \hfill
 \subcaptionbox{ }{
\includegraphics[height = 0.4\columnwidth, keepaspectratio = true]{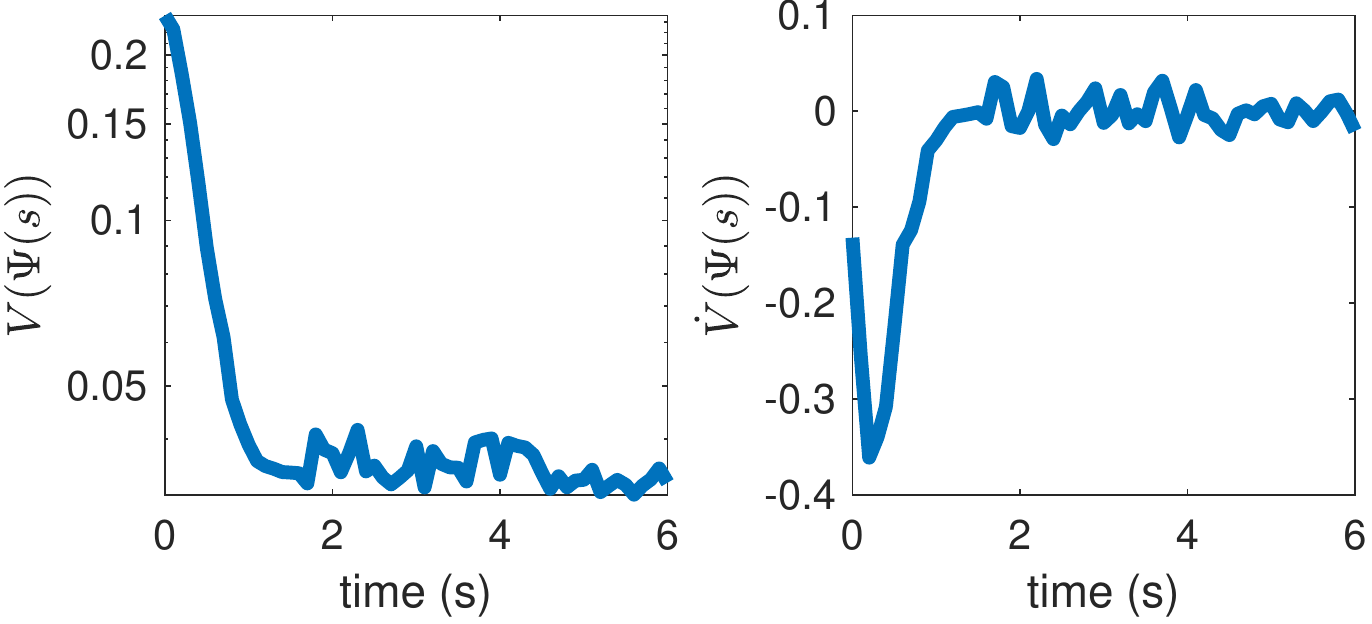}}
   \subcaptionbox{ }{
 \includegraphics[height = 0.4\columnwidth, keepaspectratio = true]{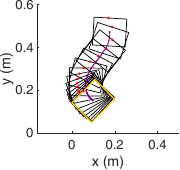}}
 \hfill
 \subcaptionbox{}{
\includegraphics[height = 0.4\columnwidth, keepaspectratio = true]{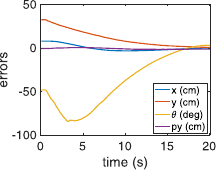}}
 \hfill
 \subcaptionbox{ }{
\includegraphics[height = 0.4\columnwidth, keepaspectratio = true]{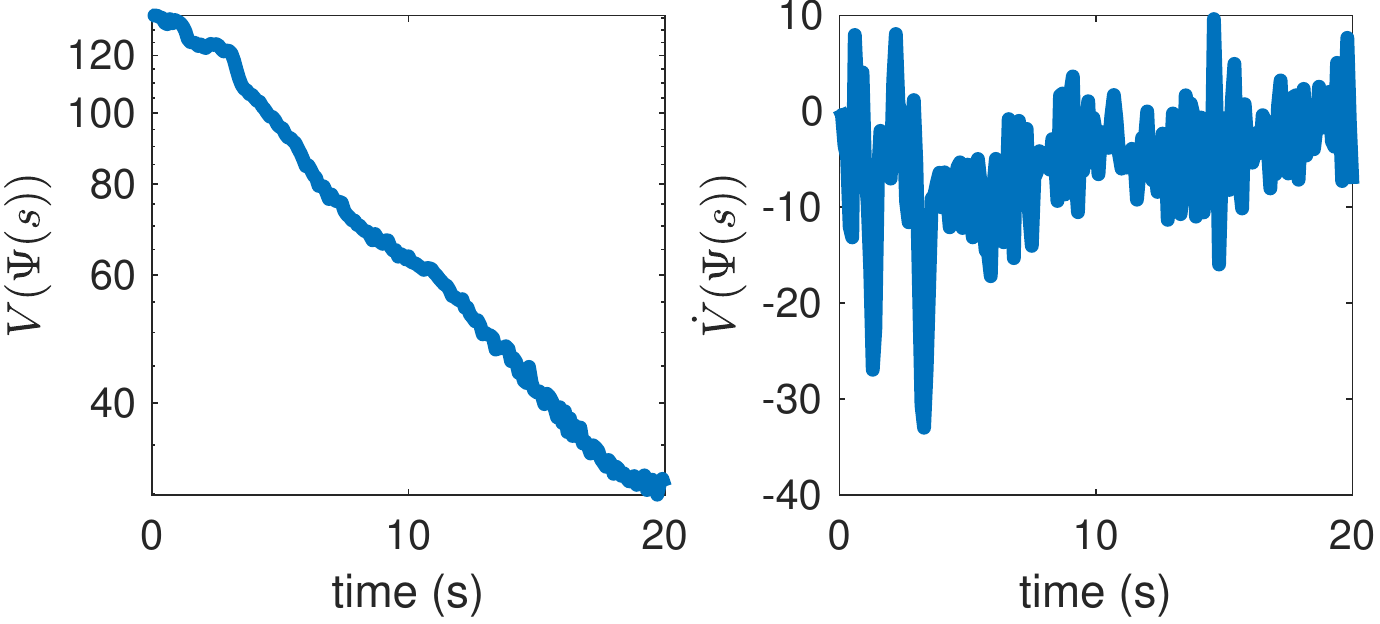}}
\caption{Control of the pusher-slider system using DISKO. We use infinite-horizon LQR control to push the block to the desired orientation, marked with yellow border. We mark the pusher and the center of the block with red and purple dots, respectively. Each row corresponds to one experimental run and shows the trajectory, the tracking errors, and the constructed candidate control-Lyapunov function that verifies that the controlled system converges to the target. The candidate control-Lyapunov function is $V(\Psi(s)) = \Psi(s)^TP\Psi(s)$, where $P$ is the solution to the Lyapunov equation \eqref{eq:: LyapunovEquation} using the stable Koopman operator. \label{fig:: pusherSlide_LQRcontrol}}
\end{figure*}
  
 \begin{figure}
     \centering
     \begin{subfigure}[]{0.49\columnwidth}
            \centering
          \includegraphics[width=0.9\linewidth]{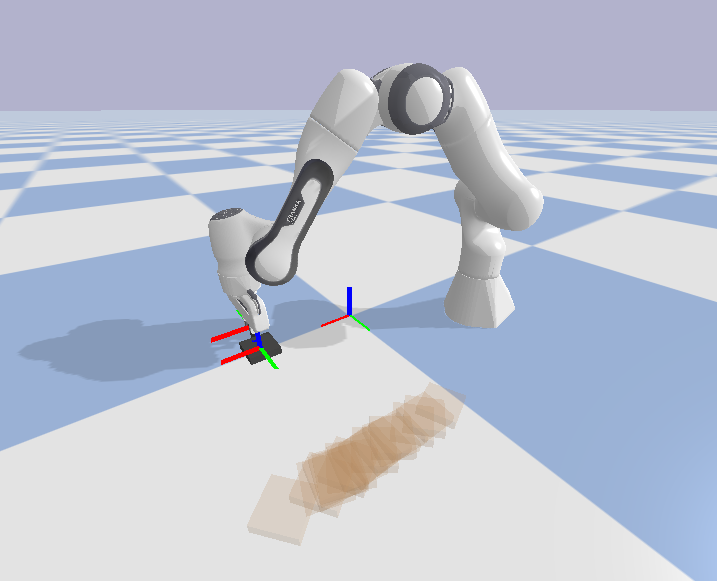}
          \caption{Simulated environment.}
           \label{table:controlComparison-a}
     \end{subfigure}
     \begin{subtable}[]{0.49\columnwidth}
     \centering
\begin{tabular}{@{}lll@{}}
\toprule
                                                                & \begin{tabular}[c]{@{}l@{}}Stable \\ Koopman\end{tabular} & \begin{tabular}[c]{@{}l@{}}LS \\ Koopman\end{tabular} \\ \midrule
\begin{tabular}[c]{@{}l@{}}Fin. \\ LQR\end{tabular}   & \bf{8.753}                                                    & 16.079                                                           \\
\begin{tabular}[c]{@{}l@{}} Inf. \\ LQR\end{tabular} & \bf{13.774}                                                    & 20.606                                                          \\ \bottomrule
\end{tabular}
     \caption{Mean terminal state cost}
    \label{table:controlComparison-b}
     \end{subtable}
     \caption{Model-based control comparison of the stable Koopman model with the unconstrained least squares (LS) Koopman model on the pusher-slider system. The models are trained on the same 5, 100 step trajectories and evaluated in simulation over a sequence of 10 experiments with target slider goals shown as opaque rendered blocks (a). The control methods using the stable Koopman model better maintains the block near the target end goal. In contrast, the control methods constructed with the LS Koopman model tends to deviate from the target goal state, resulting in worse terminal state costs. Code for the simulated Franka experiment is found in the abstract link.}
    \label{table:controlComparison}
 \end{figure}

Last, we test the control performance of the DISKO approach. We apply infinite-horizon LQR control, the gains of which are calculated offline once. We show the results in Fig. \ref{fig:: pusherSlide_LQRcontrol}. The Franka Emika Robot successfully pushes the block to the desired orientations. In the last of the three experiments, we applied a low-pass filter to increase the controller lag time and slow down the experiment in order to capture more of the dynamic response of the robot experiment. We formulate the LQR objective in terms of $x, y, \theta$, and $py$. The weights used for the states are respectively $Q = \text{diag}[800, 800, 50, 10 ,0,0,0,0,0,0,0,0]$ and the control weights are $R = \text{diag}[10^4, 10^4]$. In all three experiments, the pusher successfully moves the block close to the desired configuration. 

We further compare the efficacy of the stable Koopman model for control against a least squares unconstrained Koopman model trained on the same data. Here, a simulated environment that mimics Fig.~\ref{fig:: pusherSliderSetup} is used to directly compare the models using a finite and infinite horizon LQR control strategy. The experimental settings and parameters are given by $Q = \text{diag}[100, 100, 100, 1,0,0,0,0,0,0,0,0]$ and the control weights are given by $R = \text{diag}[10^4, 10^4]$. We test each controller over a 1000 step time window at 10Hz over 10 target end-goals. The goals are placed at a distance 0.4 meters away ($y$ direction) from the starting point and uniformly distributed at a range of 0.4 m ($x$ direction) and $\pi/2$ radians from the initial slider origin starting state ($0,0,0$) as shown in Fig.~\ref{table:controlComparison-a}.

As illustrated in Fig.~\ref{table:controlComparison-b}, the learned stable Koopman model improves the control performance on the pusher-slider task for both the finite and infinite horizon LQR controllers. The controllers constructed with the least squares Koopman model would on average deviate from the terminal state causing poor tracking performance compared to the stable model. The discrepency in the performance is due to the stable model better capturing the underlying dynamics of the pusher-slider system in contrast to the least squares model which can have incorrectly modeled instabilities.  A similar issue is often reported with nonlinear neural network models that are used for model-based control. Due to the nonlinearities, it is difficult to ascertain the accuracy of the models and the stability for long-time horizon prediction and control, making the neural network models difficult to control due to exploding or vanishing dynamic behavior~\cite{pal2020brief}. Imposing the stability requirements allow control synthesis to be more realizable and having effective linear models e.g., Koopman operators, enables better predictability of the underlying controlled system behavior.  

We verify that the controlled system converges to the target by constructing a candidate control-Lyapunov function with the stable Koopman operator and evaluating it with the state measurements. Note that we reconstruct the candidate control-Lyapunov function using Theorem \ref{th:: SemiGlobalStability}, using data from each individual trajectory. The candidate control-Lyapunov function is $\Psi(s(t))^T P \Psi s(t))$, where $P$ is the solution to the Lyapunov equation \eqref{eq:: LyapunovEquation} and $\mathcal{K}$ is the identified stable Koopman model. We also show the time derivative of the candidate control-Lyapunov function, $\frac{d}{dt}V(\Psi(s(t))$. Note that the candidate control-Lyapunov function is monotonically decreasing in all three cases. The fluctuations seen for $\frac{d}{dt}V(\Psi(s(t))$ at the end of the first two trials and throughout the third trial are observed because the motion of the system is slow; the measurements do not change enough and noise becomes more significant. At the end of the first two trials, the slow motion arises as the system converges to the target, while the motion is slower in the third experiment due to the low-pass filter. This is a candidate control-Lyapunov function for a basis of attraction surrounding the equilibrium, but the robot would need to collect data verifying the candidate control-Lyapunov function. This active phase of stability verification will be considered in future work.

\section{Discussion}\label{sec:: Discussion}
This paper demonstrates the benefits of stable Koopman operators in the prediction and control of data-driven nonlinear systems. We derive a formula for the prediction error associated with Koopman operators for an arbitrary number of time steps, which we use to show how unstable eigenvalues exponentially amplify modeling errors. We also derive properties for basis functions that are consistent with a stable Koopman operator and can improve the learning process. Appropriate basis functions can lead to low training error, which is in turn useful for the construction of Lyapunov functions, also shown in this work. 

Using the examples of the pendulum and the quadrotor, we demonstrate how stability constraints on the Koopman models makes them more robust to limited training data (which often arises in time-urgent tasks such as stabilizing unknown dynamics), as well as the prediction horizon. In fact, stability constraints improve even models that are computed with rich data obtained through active learning methods. Last, using the pusher-slider system, we validate the performance of DISKO experimentally. 

Learning Koopman models for prediction and control is an active research topic with many recent noteworthy results \cite{brunton2017chaos, korda2020optimal, azencot2020forecasting, lange2021fourier}. In this work, for the purposes of highlighting the importance of stability for arbitrary Koopman models without system- or application-specific knowledge, we do not exploit any of the available tools for synthesizing more accurate Koopman models. Identifying appropriate Koopman basis functions, either to obtain Koopman invariant subspaces or to improve the accuracy of the approximate models, remains a challenge and an active research area \cite{Haseli_EDMDtunableAccuracy}. The proposed DISKO approach can be used as a tool both for choosing basis functions and improving the accuracy and control of Koopman models. The conditions on the basis functions shown in Proposition \ref{prop: conditions_for_Koopman_basis} can assist in the selection of basis functions for the data-driven discovery of Koopman invariant subspaces or the improvement of approximate models. It can also complement other algorithms on Koopman prediction and control to avoid the undesirable effects of instability, especially in the low-data limit. 

Note that minimizing a multi-step prediction error \cite{koopman_deeplearning}, as well as including delay measurements in the basis functions (using Hankel matrices) \cite{Hankel_DMD} are two ways that have helped improve long-term prediction of Koopman operators. Although the resulting model is improved, there is still no guarantee that it will be stable. Further, training over the multi-step prediction error or including delay measurements can be prohibitive for online applications. Specifically, the former method is computationally more expensive and does not scale well as the number of prediction steps used in training increases. In addition, both of these alternative methods need longer series of data, which may not be feasible for time-critical learning tasks that rely on few data. Storing a longer history of the measurements also makes these solutions more memory-intensive compared to the memory-efficient running-sum approach presented in this work (see Appendix \ref{App:: Equivalence}).

This work is well-aligned with current efforts on physics-based learning methods that impose physical assumptions on data-driven models \cite{drgovna2021physics, ahmadi2020learning, kashinath2021physics}, shown to improve data-efficiency, prediction, and control performance. Several of these studies focus on stability constraints and bounding the model predictions \cite{kaptanoglu2021promoting, revay2021recurrent, jongeneel2021efficient, sashidhar2021bagging, yin2021imitation}. In addition, recent results establishing a relationship between stability in the Koopman framework and the existence of a contraction metric for nonlinear systems \cite{yi2021equivalence} makes the proposed algorithm a useful tool in the stability analysis and control of nonlinear systems. 

There are many avenues for future work. Investigating whether the derived formula for the prediction error of Koopman operators can be used as a metric to evaluate and compare different representations could help determine the best approximate finite-dimensional Koopman model. Further, we are interested in using the proposed methodology to construct Lyapunov functions for nonlinear systems and investigate the stability of nonlinear systems. These results could provide candidate control Lyapunov functions for finding stabilizing feedback laws of unknown nonlinear dynamics. Last, we want to explore the benefits of DISKO for Koopman operators used to optimize training of neural network methods \cite{optimizingNNwithKoopman}. 

\section{Acknowledgements}
This work was supported by the National Science Foundation (IIS-1717951, IIS-1837515). Any opinions, findings, and conclusions or recommendations expressed in this material are those of the authors and do not necessarily reflect the views of the National Science Foundation.

\begin{appendices}

\section{Necessary Conditions for Koopman Basis Functions Based on System Properties}\label{appendix: Necessary conditions for Koopman Basis Functions}
\subsubsection{Proof for first condition in Proposition \ref{prop:  conditions_for_Koopman_basis}}\mbox{\\}\\

\begin{proof}
Using chain rule, the Koopman dynamics can be expressed as 
\begin{align}
    \frac{d}{dt} \Psi(s(t)) = \frac{d\Psi(s(t))}{ds(t)} \mathrm{f}(s(t)).
\end{align}
Then, given $s_e$ an equilibrium point of the original system dynamics $\mathrm{f}$, such that $\mathrm{f}(s_e) = 0$, it is also an equilibrium of the Koopman dynamics. That is,
\begin{flalign*}
  && \mathrm{f}(s_e) =& 0&\\ 
  \Longrightarrow && \frac{d\Psi(s_e)}{ds_e} \mathrm{f}(s_e) =& 0&\\
  \Longrightarrow && \frac{d}{dt} \Psi(s_e) =& 0&\\
  \Longrightarrow && f(\Psi(s_e)) =& 0&
\end{flalign*}

\end{proof}
\subsubsection{Proof for second condition in Proposition \ref{prop:  conditions_for_Koopman_basis}}\mbox{\\}\\
\begin{proof}
Note that
\begin{flalign}\label{eq: lyap_stable_1}
 \notag  &&\| s(0) - s_e \|  < & \delta_s&\\ 
  \Longrightarrow && L_\Psi \| s(0) - s_e \| <& L_\Psi \delta_s&\\
  \notag \Longrightarrow && \| \Psi(s(0)) - \Psi(s_e) \| <& L_\Psi \delta_s.
\end{flalign}
Similarly, 
\begin{flalign}\label{eq: lyap_stable_2}
  \notag &&\| s(t) - s_e \|  < & \epsilon_s&\\ 
  \Longrightarrow && L_\Psi \| s(t) - s_e \| < & L_\Psi \epsilon_s&\\
  \notag \Longrightarrow && \| \Psi(s(t)) - \Psi(s_e) \| <& L_\Psi \epsilon_s.
\end{flalign}

Then, assume that for every $\epsilon_s > 0$ there exists $\delta_s >0$ such that 
\begin{flalign*}
  &&\| s(0) - s_e \|  <& \delta_s&   \Longrightarrow && \| s(t) - s_e \|  <& \epsilon_s~\forall~t\ge 0.&
\end{flalign*}
From \eqref{eq: lyap_stable_1} and \eqref{eq: lyap_stable_2} it follows that, for every $\epsilon_s > 0$ there exists $\delta_s >0$ such that 
\begin{flalign*}
  &&\| \Psi(s(0)) - \Psi(s_e) \|  <& L_\Psi\delta_s&   
  \\\Longrightarrow && \| \Psi(s(t)) - \Psi(s_e) \| <& L_\Psi \epsilon_s~\forall~t\ge 0.
\end{flalign*}
Then, it follows that, for every $\epsilon_\Psi = L_\Psi \epsilon_s > 0$ there exists $\delta_\Psi = L_\Psi\delta_s >0$ such that 
\begin{flalign*}
  &&\| \Psi(s(0)) - \Psi(s_e) \|  <& \delta_\Psi&   
  \\\Longrightarrow && \| \Psi(s(t)) - \Psi(s_e) \| <& \epsilon_\Psi~\forall~t\ge 0.
\end{flalign*}
\end{proof}

\subsubsection{Proof for third condition in Proposition \ref{prop:  conditions_for_Koopman_basis}}\mbox{\\}\\
\begin{proof}
Let $s_e$ be an asymptotically stable equilibrium for the nonlinear dynamical system \eqref{eq:: nonlinear_dynamics_states} in a region of the state space $\mathcal{D} \subseteq \mathcal{S}$. That is, $s_e$ is a Lyapunov stable equilibrium and there exists $\delta_s$ such that
\begin{align*}
     \|s(0) - s_e \| < \delta_s \Longrightarrow  \lim_{t\to\infty} \|s(t) - s_e\| = 0.
 \end{align*}

First, from the second condition of Proposition \ref{prop:  conditions_for_Koopman_basis}, if $s_e$ is a Lyapunov-stable equilibrium for the nonlinear dynamics \eqref{eq:: nonlinear_dynamics_states}, it is also a Lyapunov-stable equilibrium for the Koopman dynamics \eqref{eq:: nonlinear_dynamics_Koopman}. 
Second, from Lipschitz continuity, 
\begin{align*}
      &\| \Psi(s(0)) - \Psi(s_e) \| \frac{1}{L_\Psi} \le \|s(0) - s_e \| < \delta_s
      \\ \Longrightarrow & \| \Psi(s(0)) - \Psi(s_e) \| < L_\Psi \delta_s = \delta_\Psi.
 \end{align*}
  In addition, 
 \begin{align*} 
     &\lim_{t\to\infty} \|s(t) - s_e\| = 0 
     \\ \Longrightarrow &\lim_{t\to\infty} \|\Psi(s(t)) - \Psi(s_e)\| = 0.
 \end{align*}
 Therefore, $s_e$ is a Lyapunov-stable equilibrium for the Koopman dynamics and there exists $\delta_\Psi$ such that
 \begin{align*}
     \| \Psi(s(0)) - \Psi(s_e) \| < \delta_\Psi \Longrightarrow \lim_{t\to\infty} \|\Psi(s(t)) - \Psi(s_e)\| = 0.
 \end{align*}
\end{proof}
\section{Equivalent Matrix Representation for Sum of Squares Error}\label{App:: Equivalence}
Note that the Frobenius norm of $A \in \mathbb{R}^{m \times n}$ is
\begin{align*}
    \lVert A \rVert_F = \sqrt{\sum_{i = 1}^m \sum_{j=1}^n |a_{ij}|^2}.
\end{align*}
Let $\tilde{\mathcal{K}}_d \in \mathbb{R}^{W \times W}$ and $\Psi(\cdot) \in \mathbb{R}^W$ and consider the expression
\begin{align}\label{eq:: Appendix_SumLS}
    \sum_{k = 1}^{P}\lVert \Psi(s(t_k + \Delta t), u(t_k + \Delta t)) - \tilde{\mathcal{K}}_d \Psi(s(t_k), u(t_k))\rVert^2.
\end{align}
Then, using the Frobenius norm definition for the vector $\Psi(s(t_k + \Delta t), u(t_k + \Delta t)) - \tilde{\mathcal{K}}_d \Psi(s(t_k), u(t_k)) \in \mathbb{R}^{W}$,
\begin{align*}
\eqref{eq:: Appendix_SumLS} =& 
\sum_{k = 1}^{P} (\sum_{i=1}^W |\Psi_i(s(t_k + \Delta t), u(t_k + \Delta t)) \\&- \tilde{\mathcal{K}}_{d_i} \Psi_i(s(t_k), u(t_k))|^2,
\end{align*}
where $\tilde{\mathcal{K}}_{d_i} \in \mathbb{R}^{W}$ is a row vector that corresponds to the $i$th row of $\tilde{\mathcal{K}}_{d}$. Then, consider the term inside the absolute value as the $\mathcal{G}_{ik}$ element of a matrix $\mathcal{G} \in \mathbb{R}^{W \times P}$ such that
\begin{align*}
\eqref{eq:: Appendix_SumLS} =& 
\sum_{k = 1}^{P} (\sum_{i=1}^W |\mathcal{G}_{ij}||^2.
\end{align*}
We can express $\mathcal{G}$ as 
\begin{equation*}
    \resizebox{0.99\hsize}{!}{$
\begin{aligned}
\mathcal{G}^T = 
    \begin{bmatrix}
        (\Psi(s(t_1 + \Delta t), u(t_1 + \Delta t)) - \tilde{\mathcal{K}}_d \Psi(s(t_1), u(t_1)))^T \\
         \vdots 
         \\ 
         (\Psi(s(t_P + \Delta t), u(t_P + \Delta t)) - \tilde{\mathcal{K}}_d \Psi(s(t_P), u(t_P)))^T
    \end{bmatrix}
\end{aligned}
   $}
\end{equation*}
and rewrite it as 
\begin{equation*}
\resizebox{0.99\hsize}{!}{$
\begin{aligned}
    \mathcal{G}^T =
    \begin{bmatrix}
    \Psi(s(t_1 + \Delta t), u(t_1 + \Delta t)) ^T\\
    \vdots \\
    \Psi(s(t_P + \Delta t), u(t_P + \Delta t)) ^T
    \end{bmatrix}
    - 
    \begin{bmatrix}
    \Psi(s(t_1), u(t_1))) ^T
    \\ 
    \vdots 
    \\
    \Psi(s(t_P), u(t_P))) ^T
    \end{bmatrix}
    \tilde{\mathcal{K}}_d^T.
\end{aligned}$}\end{equation*}
Let $X, Y \in \mathbb{R}^{W \times P}$ be given by
\begin{equation*}\resizebox{0.99\hsize}{!}{$\begin{aligned}
X = \begin{bmatrix} \Psi(s(t_1), u (t_1 ))^T \\ \vdots \\ \Psi(s(t_P), u (t_P))^T
 \end{bmatrix}^T
 Y = \begin{bmatrix} \Psi(s(t_1 + \Delta t), u (t_1 + \Delta t))^T \\ \vdots \\ \Psi(s(t_P + \Delta t), u (t_P + \Delta t))^T
 \end{bmatrix}^T
\end{aligned}$}\end{equation*}
such that
\begin{align*}
    \mathcal{G} = Y - \tilde{\mathcal{K}}_d X.
\end{align*}
Then, using the Frobenius norm definition, we can rewrite \eqref{eq:: Appendix_SumLS} as
\begin{align*}
    \eqref{eq:: Appendix_SumLS} =& \lVert \mathcal{G} \rVert_F^2 \\
    =& \Vert Y - \tilde{\mathcal{K}}_d X \rVert_F^2.
\end{align*}

\section{Memory Preserving Gradient Descents}\label{App:: Memory-Preserving Gradient Descents}
Let $\Psi(s(t))\in \mathbb{R}^W$ be given by
\begin{align*}
\Psi(s(t)) =& \begin{bmatrix} \Psi_1(s(t)) & \Psi_2(s(t)) & \dots & \Psi_W(s(t)) \end{bmatrix}^T 
\end{align*}
and $X, Y \in \mathbb{R}^{W \times P}$ be given by 
\begin{align*}
X =& 
 \begin{bmatrix}
 \Psi(s(t_1))
 &
 \Psi(s(t_2))
 &
 \dots
 & 
 \Psi(s(t_P))
 \end{bmatrix}
 \intertext{and}
Y =& 
 \scalemath{1}{\begin{bmatrix}
 \Psi(s(t_1 +\Delta t))
 &
 \Psi(s(t_2 + \Delta t))
 &
 \dots
 & 
 \Psi(s(t_P + \Delta t))
 \end{bmatrix}}.
\end{align*}
Then,
\begin{equation*}
\resizebox{0.99\hsize}{!}
{$
\begin{aligned}
 X X^T =&
 \begin{bmatrix}
 \Psi(s(t_1))
 &
 \Psi(s(t_2))
 &
 \dots
 & 
 \Psi(s(t_P))
 \end{bmatrix}
 \begin{bmatrix}
 \Psi(s(t_1))^T \\
 \Psi(s(t_2))^T \\
 \vdots \\
 \Psi(s(t_P))^T
 \end{bmatrix}
 \\ =&
 \begin{bmatrix}
 \Psi_1(s(t_1))
 &
 \Psi_1(s(t_2))
 &
 \dots
 & 
 \Psi_1(s(t_P)) \\
 \Psi_2(s(t_1))
 &
 \Psi_2(s(t_2))
 &
 \dots
 & 
 \Psi_2(s(t_P)) \\
 \vdots & \vdots & \dots & \vdots \\
 \Psi_W(s(t_1))
 &
 \Psi_W(s(t_2))
 &
 \dots
 & 
 \Psi_W(s(t_P))
 \end{bmatrix}\\
 &\cdot
 \begin{bmatrix}
 \Psi_1(s(t_1)) & \Psi_2(s(t_1)) & \dots & \Psi_W(s(t_1)) \\
 \Psi_1(s(t_2)) & \Psi_2(s(t_2)) & \dots & \Psi_W(s(t_2)) \\
 \vdots & \vdots & \dots & \vdots \\
 \Psi_1(s(t_P)) & \Psi_2(s(t_P)) & \dots & \Psi_W(s(t_P))
 \end{bmatrix}
 \\ =&
 \begin{bmatrix}
 \sum\limits_{k=1}^P \Psi_1(s(t_k)\Psi_1(s(t_k) & \dots & \sum\limits_{k=1}^P \Psi_1(s(t_k) \Psi_W(s(t_k) 
 \\
 \vdots & \ddots & \vdots \\
 \sum\limits_{k=1}^P \Psi_W(s(t_k)\Psi_1(s(t_k) & \dots & \sum\limits_{k=1}^P \Psi_W(s(t_k) \Psi_W(s(t_k)
 \end{bmatrix}\\
 =& 
 \sum\limits_{k=1}^P
 \begin{bmatrix}
 \Psi_1(s(t_k)\Psi_1(s(t_k) & \dots & \Psi_1(s(t_k) \Psi_W(s(t_k) 
 \\
 \vdots & \ddots & \vdots \\
 \Psi_W(s(t_k)\Psi_1(s(t_k) & \dots & \Psi_W(s(t_k) \Psi_W(s(t_k)
 \end{bmatrix}\\
 =& 
 \sum\limits_{k=1}^P
 \Psi(s(t_k)\Psi(s(t_k)^T \\
 =& \mathcal{G}.
\end{aligned}
$}
\end{equation*}
Similarly, $YX^T = \mathcal{A}$, $XU^T = X_U$, $YU^T = Y_U$, and $UU^T = U_U$.

\section{Stability-Based Filtering of Basis Functions}\label{app:: FilteringOfBasisFunctions}
Consider the pendulum dynamics with equilibrium points at $[\theta, \dot\theta]$ = $[n\pi, 0]$ for $n = \mathcal{Z}$. Also consider basis functions that include the system states and $k$ functions of the states $g$ such that $\Psi(s) = [\theta, \dot\theta, g(\theta, \dot\theta)]^T$. Choose $g(\theta, \dot\theta)$ such that all but one evaluate to zero at the equilibrium points, that is $g_i(\theta_e, \dot\theta_e) = 0$ for all $i \ne k$. Then, from Proposition \ref{prop: conditions_for_Koopman_basis}, the Koopman dynamics share the same equilibrium point as the original dynamics. As a result, for a Koopman operator that captures the dynamics without loss of accuracy it must be true that at the equilibrium point $[\theta, \dot\theta]= [0, 0]$, $\frac{d}{dt}\Psi(s_e(t)) = 0$, such that
\begin{align*}
    \begin{bmatrix}0 \\ 0 \\ \vdots \\ 0 \end{bmatrix} = \begin{bmatrix}\tilde{\mathcal{K}}_{11} & \tilde{\mathcal{K}}_{12} & \cdots & \tilde{\mathcal{K}}_{1W} 
    \\ \tilde{\mathcal{K}}_{21} & \tilde{\mathcal{K}}_{22} & \cdots & \tilde{\mathcal{K}}_{2W} 
    \\ \vdots & \vdots & \cdots & \vdots 
    \\
    \tilde{\mathcal{K}}_{W1} & \tilde{\mathcal{K}}_{W2} & \cdots & \tilde{\mathcal{K}}_{WW} 
    \end{bmatrix}
    \begin{bmatrix}0 \\ 0 \\ \vdots \\ g_k(0, 0) \end{bmatrix}
\end{align*}
Then, for each row $i$, 
\begin{align*}
    0 =& \tilde{\mathcal{K}}_{iW}g_k(0, 0) \\
   \Longrightarrow 0 =&\tilde{\mathcal{K}}_{iW}.
\end{align*}
That means that, for a Koopman model that captures the dynamics with full fidelity and using the basis functions shown here, the function $g_k(\theta, \dot\theta)$ is redundant. After removing $g_k(\theta, \dot\theta)$ from $\Psi(s)$ and repeating this analysis at the equilibrium  $[\theta, \dot\theta]= [\pi, 0]$, one finds that $\theta$ is also redundant. Leveraging the conditions for the Koopman models as shown in Proposition \ref{prop: conditions_for_Koopman_basis} to optimize the choice of basis functions is a promising research topic that merits further investigation and is deferred to future work.
\end{appendices}

\bibliographystyle{IEEEtran}
\normalsize{\normalsize\bibliography{references}}

\end{document}